\documentclass[twoside]{article} %
\usepackage[accepted]{aistats2024}

\usepackage{amsmath,amsfonts,bm}

\def\eqref#1{equation~\ref{#1}}

\def\1{\bm{1}}

\DeclareMathAlphabet{\mathsfit}{\encodingdefault}{\sfdefault}{m}{sl}
\SetMathAlphabet{\mathsfit}{bold}{\encodingdefault}{\sfdefault}{bx}{n}

\usepackage{amsmath,amsfonts}
\usepackage{algorithmic}
\usepackage{algorithm}
\usepackage{array}
\usepackage{dirtytalk}
\usepackage[caption=false,font=normalsize,labelfont=sf,textfont=sf]{subfig}

\usepackage{natbib}
\usepackage{relsize}
\usepackage{textcomp}
\usepackage{stfloats}
\usepackage{hyperref}
\usepackage{url}
\usepackage{cleveref}
\usepackage{verbatim}
\usepackage{graphicx}
\usepackage{cite}
\usepackage{xcolor}
\usepackage{paralist}
\usepackage{stfloats}
\usepackage{tabularx}
\usepackage{bm}
\usepackage[tableposition=bottom]{caption}

\usepackage{multirow}
\usepackage{booktabs}
\usepackage{tikz}
\usepackage{amsthm}
\usepackage{thmtools} 
\usepackage{thm-restate}
\usepackage{wrapfig}

\newtheorem{proposition}{Proposition}

\newtheorem{remark}{Remark}
\newtheorem{lemma}{Lemma}

\usepackage{color}
\usepackage{xspace}
\newcommand{\DEHNN}  {{\sf{DE-HNN}\xspace}}
\newcommand{\myparagraph}[1] {{\vspace*{0.07in}\noindent{\bf #1~}}}

\newcommand{\ahyperE} {{\sigma}}
\newcommand{\sinkset}  {{\mathsf{S}}}
\newcommand{\acell}    {{\mathsf{c}}}
\newcommand{\anet}  {{\sigma}}
\newcommand{\myC}  {{\mathcal{C}}}
\newcommand{\myNets} {{\mathcal{N}}}
\newcommand{\adH}  {{\overrightarrow{{H}}}}
\newcommand{\hE}   {{\overrightarrow{\Sigma}}}
\newcommand{\NN}  {{\mathcal{I}}}
\newcommand{\dr}   {{\mathsf{v}}}
\newcommand{\source} {{\dr}}
\newcommand{\anetlist}  {{\mathcal{H}}}
\newcommand{\avn}   {{\omega}}

\newcommand{\Fcal}  {{\mathcal{F}}}
\newcommand{\denselist}{\itemsep 0pt\parsep=1pt\partopsep 0pt}
\newcommand{\msetL} {{\{\!\{}}
\newcommand{\msetR} {{\}\!\}}}

\begin{document}

\runningauthor{Luo, Hy, Tabaghi, Defferrard, Rezaei, Carey, Davis, Jain, Wang}

\twocolumn[

\aistatstitle{\DEHNN: An effective neural model for Circuit Netlist representation}

\aistatsauthor{ 
Zhishang Luo$^1$ \And 
Truong Son Hy$^2$ \And  
Puoya Tabaghi$^1$ \And
Donghyeon Koh$^4$ \And
Michael Defferrard$^4$ \AND
Elahe Rezaei$^3$ \And
Ryan Carey$^3$ \And
Rhett Davis$^5$ \And
Rajeev Jain$^3$ \And
Yusu Wang$^1$
}

\aistatsaddress{ 
$^1$ University of California San Diego \And 
$^2$ Indiana State University \AND
$^3$ Qualcomm Technologies, Inc. \And
$^4$ Qualcomm Wireless GmbH \And
$^5$ North Carolina State University
} 
]

\begin{abstract}
     The run-time for optimization tools used in chip design has grown with the complexity of designs to the point where it can take several days to go through one design cycle which has become a bottleneck. Designers want fast tools that can quickly give feedback on a design. Using the input and output data of the tools from past designs, one can attempt to build a machine learning model that predicts the outcome of a design in significantly shorter time than running the tool. The accuracy of such models is affected by the representation of the design data, which is usually a netlist that describes the elements of the digital circuit and how they are connected. Graph representations for the netlist together with graph neural networks have been investigated for such models.  However, the characteristics of netlists pose several challenges for existing graph learning frameworks, due to the  large number of nodes and the importance of long-range interactions between nodes. To address these challenges, we represent the netlist as a directed hypergraph and propose a  \emph{Directional Equivariant Hypergraph Neural Network} (\DEHNN) for the effective learning of (directed) hypergraphs. Theoretically, we show that our \DEHNN{} can universally approximate any node or hyperedge based function that satisfies certain permutation equivariant and invariant properties natural for directed hypergraphs.  We compare the proposed \DEHNN{} with several State-of-the-art (SOTA) machine learning models for (hyper)graphs and netlists, and show that the \DEHNN{} significantly outperforms them in predicting the outcome of optimized place-and-route tools directly from the input netlists. Our source code and the netlists data used are publicly available at \url{https://github.com/TILOS-AI-Institute/DEHNN.git}.

\end{abstract}

\section{Introduction}
\label{sec:Intro}

Chip design is a complicated process involving numerous steps, many of which involve solving hard optimization problems. Just consider the stage of the \emph{place and route} of a synthesized netlist:  Here the input is a netlist consisting of \emph{cells} and \emph{nets}, where cells refer to functional units such as logic gates, and nets refer to connections between cells. The goal is to produce a layout of this netlist in a specific 2D region, where gates are placed and connections among them are realized by wires laid out across multiple layers (called ``routed''), all while aiming to optimize multiple key properties (e.g, minimizing total wirelength and reducing congested ``hotspots'').  This place-and-route stage is highly nontrivial to solve for large netlists, and requires a time-consuming process in practice with multiple stages and iterations. There therefore arises the need for data-driven methods to predict properties of a design directly without the time-consuming place and routing process. To this end, graph neural networks become natural choices, given that the netlists are often represented
as a graph or a hypergraph. In this paper, we aim to develop an efficient and effective graph learning architecture to predict post-routing properties (e.g., wirelength or congestion) for a synthesized netlist accurately. The past decade has witnessed a tremendous growth of graph learning models. Two most popular families are  (i) message-passing neural networks (MPNNs) \citep{10.5555/3305381.3305512, Jegelka22}, and (ii) transformer based approaches (e.g, survey \citep{mueller2023attending}). However, netlists present several challenges for existing graph learning architectures: (i) their size can be massive, from hundreds of thousands to millions of nets, (ii) long-range interactions are important (e.g., properties of interest might be caused by long paths and other long-range interactions), and (iii) properties of post-routing netlists seem to depend on complex information of graph topology beyond simple statistics such as in-/out-degrees (distributions). 

Unfortunately, it is challenging for popular MPNNs to capture long-range interactions, due to issues such as over-smoothing of graph signals \citep{Chen_Lin_Li_Li_Zhou_Sun_2020} and oversquashing bottlenecks \citep{topping2022understanding}. MPNN's ability in capturing graph motifs (e.g cycles and trees) and higher-order structures is also limited \citep{xu2018how,Jegelka22,hy2019covariant}.  Transformer-based graph neural models appear to be more effective in capturing long-range interactions \citep{NEURIPS2022_8c3c6668, 10.1063/5.0152833}. However, each transformer layer typically takes time quadratic to the number of nodes. While there are sparse transformers \citep{10.5555/3524938.3525416, 10.1145/3530811}, their representation power is also reduced. Furthermore, the primary ways for a graph transformer to capture input graph topology have been either via initial position/structure encoding of nodes, or via the use of certain pairwise graph distances to ``reweight'' the attention \citep{zhang2023rethinking}. In general, it is not clear how sensitive a transformer based model is to features in input graph topology (e.g., specific paths which can be important to netlists properties).   

\myparagraph{Our work.}
In this paper, similar to \citep{10.1145/3394885.3431562}, we model a netlist as a directed hypergraph and present a novel \emph{Directional Equivariant Hypergraph Neural Networks} (\DEHNN) for the effective learning of (directed) hypergraphs. In particular, \DEHNN{} can be used to predict properties (e.g., congestion or net-wirelength) of a post-routed netlist directly from an input netlist {\bf before} performing the lengthy place-and-route process. 
Our \DEHNN{} incorporates several new ideas to address the aforementioned challenges posed by netlists. 
Our contributions are as follows: 

\begin{compactitem}   

    \item We advocate for the modeling of a netlist as a directed hypergraph. Indeed, a net usualy consists of a driver gate/cell $c$, togehter with a set $S$ of ``sinks''; see Section~\ref{sec:background} and Figure~\ref{fig:netlist}. Recognizing the difference between the driver gate and sinks in the timing of a routed net, inspired by \citep{10.1145/3394885.3431562},  we represent a net as a \emph{directed-hyperedge} $(c, S)$ to separate the roles of driver and sinks cells. 

    \item We propose a learning model \DEHNN{} for directed hypergraphs. Theoretically, we show (Theorem \ref{thm:nested_permutation_invariant}) that our \DEHNN{} can universally approximate any node or hyperedge based function that satisfies certain permutation equivariant and invariant properties natural for directed hypergraphs. 
   
\item  On the practical front, to mitigate the issue of large size of and long-range interactions in input netlists, we use a hierarchy of virtual nodes (VNs), which provides additional ``bridges'' to allow the integration of both local and global information while still maintaining original graph topology (unless in a graph pooling approach); see Figure~\ref{fig:hvn} and Section~\ref{sec:DEHNN}.  

    \item To make our initial node features more informative, in addition to using Laplacian eigenvectors to provide position encoding as in \citep{NEURIPS2022_5d842367, NEURIPS2022_5d4834a1}, we also use a topological summary called persistent homology \citep{EH10,DW22}, which can be used to encode the ``shape'' of graph motif {\bf around each node} in a multi-scale manner (e.g., \citep{zhao2020persistence,yan2021link}).  
\item We compare our \DEHNN{} with several SOTA machine learning models for (hyper)graphs and netlists. Our model significantly outperforms them in predicting different properties of post-routed netlists directly from input netlists. 

We provide careful ablation studies to demosntrate the utilities of the use of directed hyperedge, (hierarchical) VNs and persistence-based topological summaries. Finally, we remark that ML research for chip design currently suffers from the scarcity of open-source benchmark datasets\footnote{Previous netlist property prediction work sometimes releases the input netlist designs, which our paper also uses. However, they do not release the resulting properties nor the post place-and-route netlists due to the use of commercial tools.}. We hope our datasets (which will be made publicly available at \url{https://github.com/TILOS-AI-Institute/DEHNN.git}) can help bridge this gap. These netlists (of sizes from 400K to 1.3M) can also serve as benchmark for long-range graph interactions for machine learning researchers. 
\end{compactitem}

While in this paper, we design \DEHNN{} with the goal of netlists representation and learning, our architecture as well as its theoretical results are general and applicable to any directed hypergraphs.  

\myparagraph{Related work on machine learning models for netlists.}

Earlier machine learning (ML) approaches for netlist property (e.g, congestion) prediction assume that the placement of cells (logic gates) are already given, that is, the input are placed, but not yet routed. They then convert the input to a 2D image or other 2D grid-based representations to predict the final congestion map over this region using models such as convolutional neural networks \citep{8465835,10.1145/3400302.3415662,8587655, 10.1145/3465373, 10.1145/3372780.3375560, 10158384,9045178,9467312,10.1145/3489517.3530675}. 
However, the placement information is itself very time-consuming to obtain. Furthermore, representing the input as a 2D image makes it hard to capture local and global connectivity information in the netlists.

Since a circuit is represented more accurately as a (hyper)graph, recent work deploys graph neural networks (GNNs) for congestion prediction. %
The work of \citep{8920342} constructs a homogeneous 
graph representation of netlist in which each node corresponds to a cell, and if two cells are connected by a net then there exists an edge between those nodes in the graph, and then applies Graph Attention Networks (GAT) \citep{gat2018graph}. 
Later follow-up work includes using node embedding computed from partitioned subgraph (to capture more global graph structure) \citep{10.1109/ICCAD51958.2021.9643446} and using dual graph with both cell and net features \citep{10.1145/3394885.3431562}. 
Note that this convertion of a net to a clique can lead to very large sized cliques, and also loses net-specific topological information. 
The SOTA
approach for netlists representation is 
\citep{NEURIPS2022_7fa54815}, which introduces a heterogeneous (i.e., different edge types) graph construction, called \emph{circuit graph}, in which both cells and nets are represented as nodes of a bipartite graph. 

All these approaches still have the issues of capturing long-range interaction essentially for netlists. As we describe in ``Our contributions'', our \DEHNN{} deploys a suitable architecture, hierarchical virtual nodes, as well as informative persistent homology features, to build an effective graph learning model for netlists (and other directed hypergraphs). 
Note that similar to \citep{8920342,NEURIPS2022_7fa54815}, our \DEHNN{} performs learning on netlists {\bf without} placement information. If placement informtion is available, they can easily be added to initial node position encoding.

\section{Modeling netlists as directed hypergraphs}
\label{sec:background}

\paragraph{Circuit netlists.}
A circuit netlist is a textual representation of electronic components, such as logical boolean gates, and the connections between them. 
A (pre-placed, also called synthesized) netlist $\anetlist$ consists of a collection of {\bf cells} (logic gates) $\myC = \{\acell_1, \ldots, \acell_n\}$, and a set of {\bf nets} $\myNets = \{\anet_1, \ldots, \anet_m \}$. 
Each cell (gate) has a certain number of input pins and an output pin. The number of input pins is decided by the type of this gate (e.g., an AND gate takes two inputs). For every gate, its output will flow into the input pins of a collection of other gates. 
This information is captured by the concept of {\bf net}, where a net $\anet \in \myNets$ consists of the output pin of a certain cell, called its \emph{driver cell} denoted by $\dr_\anet \in \myC$, together with all those \emph{sink cells}, denoted by $\sinkset_\anet \subseteq \myC$, where the signal from this output pin will flow into. 
In other words, a net can be represented by a tuple $\anet = (\source_\anet, \sinkset_\anet)$.  %
See Figure~\ref{fig:netlist}~(a). 

Given such a netlist, standard chip design pipelines will first lay it out in the physical space (\emph{placement}). Then in the \emph{routing stage}, the connection from output pin of one cell to the input of other cells are mapped to the routing channels within the chip's physical floorplan. Among other properties, one wishes to minimize the total wirelength of each net, and to reduce routing ``congestions'', which occurs when the number of edges to be routed in a specific region of the floorplan exceeds the available routing capacity. See Figure~\ref{fig:circuit_diagram} in the Supplement for an illustration of a placed-and-routed netlist. 

\myparagraph{Directed hypergraphs.}
The standard \emph{hypergraph $H$} is a tuple $(V, \Sigma)$, where $V$ is a set of nodes, $\Sigma$ is a set of \emph{hyperedge}, and each $e\in \Sigma$ is a subset of $V$; i.e., $e\subseteq V$. 
A \emph{directed hypergraph} $\adH$ is a tuple $(V, \hE)$, where each \emph{directed hyperedge $\ahyperE \in \hE$} consists of an ordered pair $\ahyperE = (v_\ahyperE, S_\ahyperE)$ with $v_\ahyperE \in V$ and $S_\ahyperE \subseteq V$. 
It is easy to see that a netlist $(\myC, \myNets)$ can be naturally viewed as a directed hypergraph where we have: {\bf cell $\Leftrightarrow$ node}, and {\bf net $\Leftrightarrow$ directed hyperedge}. See Figure~\ref{fig:netlist} (b).

In what follows, we often use cells / nodes, as well as nets / directed hyperedges, interchangeably. In fact, for simplicity, we often use the terms ``{\bf nodes}'' and ``{\bf nets}'' as they are more concise. 
We will also refer to $\dr_\ahyperE$ and $\sinkset_\ahyperE$ from a directed hyperedge $\ahyperE = (\dr_\ahyperE, \sinkset_\ahyperE)$ as its \emph{driver} and its \emph{sinks}, respectively, just like in a net $\ahyperE$. 

Given a net $\sigma \in \hE$, we say that it {\it contains} a node $v\in V$, denoted by $v\in \sigma$, if $v$ is either the driver, or a sink of $\sigma$. 
Given a node $v\in V$, we say that a net $\sigma$ is \emph{incident on $v$} if $\sigma$ contains $v$. The collection of nets incident to $v$ is called the \emph{incident-net-set} of $v$, denoted by $\NN(v) = \{ \sigma' \in \hE \mid v \in \sigma' \}$. For example, in Figure~\ref{fig:netlist} (b), the incident-net-set of $v_3$ is $\NN(v_3) = \{\anet_1, \anet_2, \anet_5\}$. 

\begin{figure}[tp]
    \centering
\begin{tabular}{cc}
\includegraphics[height=1.8cm]{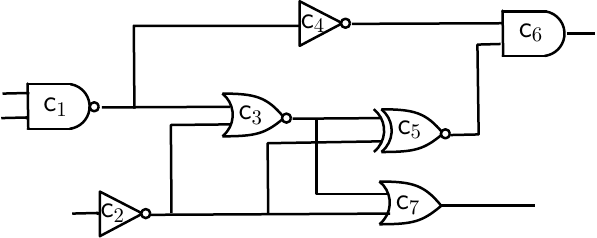} & \includegraphics[height=2.0cm]{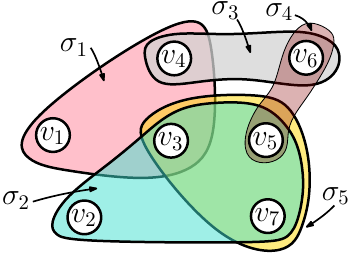} \\
(a) & (b)
\end{tabular}
    \caption{{\small (a) A netlist with $7$ cells $\myC = \{\acell_1, \ldots, \acell_7\}$ and 5 nets. For example, the output of gate $c_2$ flows into cells $c_3, c_5,$ and $c_7$, giving rise to the net $\anet = (c_2, \{c_3,c_5, c_7\})$. That is, the driver cell of $\anet$ is $\dr_{\anet} = c_2$, while its sink-set being $\sinkset_\anet = \{c_3, c_5, c_7\}$. (b) The corresponding directed hypergraph with $7$ nodes and $5$ hyperedges. Each node $v_i$ corresponds to cell $\acell_i$, and each hyperedge is marked as a shaded region.}}
    \label{fig:netlist}
\end{figure}

In the chip design literature, a netlist is oftentimes represented either (1) as a (directed) graph where two cells are connected if they belong to a common net; or (2) as a standard hypergraph where a net is a hyperedge consisting of the union of the driver cell and all sink cells. 
The former can lead to huge cliques (as some nets can consist of large number of cells) and also lose sensitivity to the net topology. 
In the learning context, the work of \citet{10.1145/3394885.3431562} first separated the role of the driver and sinks of a net in the representation of a netlist and provided justification for this choice. Their final representation is still a graph representation that is intuitively a directed version of the so-called \emph{line graph} for a hypergraph. However, converting a hypergraph to a line graph is a lossy process. The directed hypergraph provides a more informative and cleaner representation: it both preserves the full net information and differentiates between drivers and sinks.

\section{\DEHNN{}: a neural network for directed hypergraphs}
\label{sec:DEHNN}

In this section, we first describe a basic neural network (NN) model in Section~\ref{subsec:baseDEHNN}, which we refer to as base-\DEHNN{}, for the representation learning of directed hypergraphs. We provide theoretical justification of this model in Section~\ref{subsec:universalapprox}. Then in Section~\ref{subsec:augmentation}, we describe how to augment this base model to make it more effective at capturing long-range interactions as well as the multi-scale graph topology. 

\subsection{Base-\DEHNN{}}
\label{subsec:baseDEHNN}

Our base-\DEHNN{} uses message-passing mechanisms. However, it differs from standard MPNN (message-passing neural networks) \citep{10.5555/3305381.3305512} in how the messages are aggregated and updated, so as to process node and net based features, as well as to respect the direction of hyperedges. 

More specifically, base-\DEHNN{} consists of $L$ layers. Consider an input directed hypergraph $\adH = (V, \hE)$. For the $\ell$-th layer, each node $v \in V$ (resp.\ each net/directed hyperedge $\ahyperE \in \hE$) will maintain a node feature (resp.\ net feature) denoted as $m^{\ell}(v)$ (resp.\ $M^{\ell}(\ahyperE)$). For simplicity, assume that $m^{\ell}(v), M^{\ell}(\ahyperE) \in \mathbb{R}^{d_\ell}$ are $d_\ell$-dimensional vectors.  
Assume first that our final goal is to predict net properties.
The base-\DEHNN{} will compute $m^{\ell}(v)$ and $M^{\ell}(\ahyperE)$ using feature representations from the ($\ell-1$)-th layers by the following two steps: 

\noindent {\bf [Node Update]:} First, the features of each {\bf node} (cell) $v\in V$ is updated using features of the set of those {\bf nets} containing it, that is, via the features of those nets in the incident-net-set $\NN(v)$ of $v$
as follows:  
    \begin{equation}\label{eq:netlist2} 
m^{\ell}(v) =  \mathrm{Agg}^\ell_{\ahyperE \rightarrow v} ( \msetL M^{\ell - 1}(\sigma^{\prime} )\msetR_{\sigma^{\prime} \in \NN(v)} ),
\end{equation}
where $\msetL \cdot \msetR$ denotes a \emph{multiset} as some neighboring nets could have identical feature representations. The function $\mathrm{Agg}_{\ahyperE \rightarrow v}$ operates on a multiset and should be \emph{invariant} to the order of neighboring nets of $v$ in $\NN(v)$. 
Similar to the Deep Set architecture \citep{NIPS2017_f22e4747} which can hancle such \emph{permutation invariance} in multisets, we implement the function $\mathrm{Agg}_{\ahyperE \rightarrow v}$ by: 
\begin{equation} \label{eq:node-update-implementation}
m^{\ell}(v) = \sum_{\sigma^{\prime} \in \NN(v)} \text{MLP}_1^\ell\big( M^{\ell - 1}(\sigma^{\prime}) \big), 
\end{equation}
where $\text{MLP}_1$ stands for a multi-layer perceptron. For example, the update of node feature for $v_4$ in Figure~\ref{fig:netlist}~(b) is 
$m^{\ell}(v_4) = \text{MLP}_1^\ell\big( M^{\ell - 1}(\sigma_1)\big) + \text{MLP}_1^\ell\big( M^{\ell - 1}(\sigma_3)\big)$
as $\NN(v_4) = \{\anet_1, \anet_3\}$. 
It is easy to see that the update in Eqn~(\ref{eq:node-update-implementation}) satisfies the needed permutation invariance. 

\noindent{\bf [Net Update]:} Next, the features of each {\bf net} $\ahyperE = (\dr_\ahyperE, \sinkset_\ahyperE)$ is updated using the new node features for those {\bf nodes} contained in $\ahyperE$. Since the net (hyperedge) $\ahyperE$ is directed, we wish to separate the roles of its driver node $\dr_\ahyperE$ and the set of sinks $\sinkset_\ahyperE$, that is, 
\begin{equation}\label{eq:netlist}
M^{\ell}(\ahyperE) =  \mathrm{Agg}^\ell_{v \rightarrow \ahyperE}( m^\ell(\dr_\sigma), \msetL m^\ell(v^{\prime})\msetR_{v^{\prime} \in \sinkset_\ahyperE} )
\end{equation}
However, the update should {\bf not} depend on the ordering of nodes in the sink set $\sinkset_\ahyperE$, i.e., $\mathrm{Agg}_{v \rightarrow \ahyperE}$ needs to be \emph{permutation invariant} w.r.t. its second parameter, the multiset $\msetL m^{\ell}(v^\prime)) \msetR_{v^\prime \in \sinkset(\sigma)}$. %
We use the following to implement Eqn~(\ref{eq:netlist}) to guarantee the needed permutation invariance of $\mathrm{Agg}_{v \rightarrow \ahyperE}$ w.r.t. the ordering of nodes in $\sinkset_\ahyperE$:
\begin{align}\label{eq:net-update-implementation}
M^{\ell}(\ahyperE) = & \, \text{MLP}^\ell_3 \bigg[ m^{\ell}(\dr_{\ahyperE}) \oplus \bigg( \sum_{v^{\prime} \in \sinkset_\ahyperE} \text{MLP}^\ell_2(m^{\ell}(v^\prime)) \bigg) \bigg] %
\end{align}

where $\text{MLP}_2$ and $\text{MLP}_3$ are multi-layer perceptrons, and $\oplus$ denotes vector concatenation. 
For example, in Figure~\ref{fig:netlist}, the update of $\anet_1 = (v_1, \{v_3, v_4\})$ is $M^{\ell}(\ahyperE_1) = \text{MLP}^\ell_3 \big( m^{\ell}(v_1) \oplus (\text{MLP}^\ell_2(m^{\ell}(v_3)) + \text{MLP}^\ell_2(m^{\ell}(v_4)) \big)$. 

We note that if the target task is to predict node-level features (instead of net-level features), then we will apply dual update rules where the roles of nets and nodes will be swapped. Finally, in our implementation, we also add residuals (i.e, node / net features from the previous level) to each node / net features during the updates.
There are $L$ such layers, and in the end, if the given task is a regression task, then another linear layer $\psi$ is applied to the values $m^L(\cdot)$ (resp.\ $M^L(\cdot)$) to generate the final node-based (resp.\ net-based) regression function. During the training process, the loss function in this case is the standard mean squared error; that is, if it is a node-based regression with ground-true function $\mathrm{y}: V\to \mathbb{R}$, then we have: 
\begin{equation}
\mathcal{L}(\bm{\theta}) = \frac{1}{|V|} \sum_{v \in V} [\psi(m^L(v)) - \mathrm{y}(v)]^2,
\label{eq:loss}
\end{equation}
where $\bm{\theta}$ denotes the set of all learnable parameters in the entire base-\DEHNN{} model.
Node or net classification tasks are handled similarly but with the cross-entropy loss. %

\subsection{Theoretical analysis of \DEHNN{}}
\label{subsec:universalapprox}

We now provide some theoretical guarantee of our base-\DEHNN{} model. 
For simplicity, in what follows we assume that our target (regression) functions are net-valued functions\footnote{The net-valued function is more natural for properties such as net wirelength and congestion.} (or simply \emph{net-functions}) of the form $F: \hE \to \mathbb{R}$. 
Our main result below holds for node-valued functions via a symmetric argument. 

Let us consider one update stage at a fixed layer $\ell \in [1, L]$. 
From the  net-function perspective, the goal of the update stage is the following: At the beginning of this stage, we start with input net features $\mu_\anet := M^{\ell-1}(\anet)$, for all $\anet \in \hE$. In the end, we obtain a set of new features $\mu^{*}_\anet := M^{\ell}(\anet)$ for each $\anet \in \hE$. If we view this feature update as a function on net features, then ideally, for any fixed $\anet \in \hE$, its new feature should depend on the features of all its \emph{neighboring nets} which are those nets that share at least one node with $\anet$. However, note that the neighbors of $\anet$ are naturally classified to two families: 
\begin{description}\denselist 
    \item[(type-A)] those in the set $\NN(\dr_\anet)$ (recall $\NN(v)$ is the set of nets that contains node $v$), which are connected to $\anet$ via the driver node $\dr_\anet$ of $\anet$; and 
    \item[(type-B)] those in $\bigcup_{v^{\prime} \in \sinkset_\ahyperE} \NN(v^{\prime}) $, where each set $\NN(v^{\prime})$ are those nets connected to $\anet$ via a sink node $v^\prime$ of $\anet$ (i.e., $v^\prime \in \sinkset_\anet$). Note that $\Big\{ \NN(v^{\prime})  \Big\}_{v^{\prime} \in \sinkset_\ahyperE}$ is {\bf a set of sets} of neighboring (type-B) nets of $\anet$.     
\end{description}
For example, in Figure~\ref{fig:netlist}, for $\sigma_5 = \{v_3, \{v_5, v_7\}\}$, its (type-A) neighbors are $\NN(v_3) = \{\anet_1, \anet_2, \anet_5\}$, while (type-B) is a set of sets $\{ \{\anet_2,\anet_4,\anet_5\}, \{\anet_2, \anet_5\}\}$. 
It is natural to model the desired update function for the directed hyper-edges as follows, which differentiate (type-A) and (type-B) neighbors of $\anet$: 
\begin{align} \label{eq:M}
    \mu^*_\anet &= 
    \mathcal{F} \Big( \msetL \mu_{\anet'} \msetR_{\anet' \in \NN(\dr_\ahyperE)} , \Big\{\!\Big\{ \msetL  \mu_{\anet'} \msetR_{\anet' \in \NN(v^{\prime})}  \Big\}\!\Big\}_{v^{\prime} \in \sinkset_\ahyperE}  \Big) 
\end{align}

Note that the first variable, $\msetL \mu_{\anet'} \msetR_{\anet' \in \NN(\dr_\ahyperE)}$, consists of the set of input feature representations of (type-A) neighbors of $\anet$, while the second variable, $\Big\{\!\Big\{ \msetL  \mu_{\anet'} \msetR_{\anet' \in \NN(v^{\prime})}  \Big\}\!\Big\}_{v^{\prime} \in \sinkset_\ahyperE}$,
is {\bf a set of sets\footnote{For concision, we use ``set'' instead of ``multiset'' here.} (of net features)}, constituting multisets of feature representations of those (type-B) neighbors of $\anet$. 
The function $\mathcal{F}$ should be invariant not only to the order within each individual set $\NN(u)$ for some cell $u$, but also invariant to the order of nodes in the sink-set $\sinkset_\anet$, i.e., the order of sets $\{  \mu_{\anet'} \}_{\anet' \in \NN(v^{\prime})}$'s within the outer-set in the second variable of $\Fcal$. 
We refer to such a function $\mathcal{F}$ as {\bf nested-permutation invariant}.

Now recall that in our base-\DEHNN{}, at each layer we perform update of node-feature by features of its incident nets as described in Eqn~(\ref{eq:node-update-implementation}), followed by the update of net-features by the new features of those nodes contained in $\anet$ as specified in Eqn~(\ref{eq:net-update-implementation}). 
In other words, one can think of the update of net features from $M^{\ell-1}(\cdot)$ to $M^\ell(\cdot)$ to be the composition of updates in Eqns~(\ref{eq:node-update-implementation}) and (\ref{eq:net-update-implementation}). By construction, it is easy to see that the composition of these two udpate steps indeed gives rise to a nested-permutation invariant function.
However, can any nested-permutation invariant net-function be represented (or approximated) by such a composition of two steps?  
A priori, the answer to this opposite direction is not clear at all as the iterative updates factored through the node features appears more restrictive. Our main theorem below shows that the answer is in fact positive. 

\begin{restatable}[Simplified]{theorem}{goldbach}
\label{thm:nested_permutation_invariant}
Let $\Fcal$ be any continuous, nested-permutation invariant, net-value function as in Eqn~(\ref{eq:M}). For simplicity, assume both input nets and output of $M$ take values in a compact set $\mathcal{B} \subset \mathbb{R}^d$, 
    a connected compact subset of $\mathbb{R}^d$. 
    Then we have that $\Fcal$ can be rewritten as the following sum-decomposition 
\begin{align}
     \forall \sigma: ~& \Fcal\Big( \{ \mu_{\anet'} \}_{\anet' \in \NN(\dr_\ahyperE)} , 
\Big\{ \{  \mu_{\anet'} \}_{\anet' \in \NN(v^{\prime})}  \Big\}_{v^{\prime} \in \sinkset_\ahyperE}  \Big) \nonumber \\
= ~& \rho \Big( \sum_{ \sigma^{\prime} \in \NN(\dr_\ahyperE) } \phi_1 (\mu_{\sigma^{\prime}}),
        \sum_{v^{\prime} \in \sinkset_\ahyperE}\phi_2\big( \sum_{ \sigma^{\prime} \in \NN(v^{\prime}) } \phi_1 (\mu_{\sigma^{\prime}}) \big) \Big)  \label{eq:sumdecomposition}
    \end{align} 
    where $\phi_1: \mathbb{R}^d \rightarrow \mathbb{R}^{d^{\prime}}$, $\phi_2: \mathbb{R}^{d^{\prime}} \rightarrow \mathbb{R}^{d^{\prime \prime}}$, and $\rho$ are continuous functions.
\end{restatable}

Recall that for the $\ell$-th layer, the input feature for any net $\anet'$ is  $\mu_{\anet} = M^{\ell-1}(\anet')$.
Now compare the right-hand side of Eqn~(\ref{eq:sumdecomposition}) with Eqns~(\ref{eq:node-update-implementation}) and (\ref{eq:net-update-implementation}): it is easy to see that by using $\text{MLP}_1$ from Eqn~(\ref{eq:node-update-implementation}) to approximate the continuous function $\phi_1$, and using $\text{MLP}_2$ and $\text{MLP}_3$ from Eqn~(\ref{eq:net-update-implementation}) to approximate continuous functions $\phi_2$ and $\rho$ respectively, we then have that our iterative updates using Eqns~(\ref{eq:node-update-implementation}) and (\ref{eq:net-update-implementation}) approximate any desired update of the net features via a nested-permutation invariant function as in Eqn~(\ref{eq:M}). In other words, the iterative message-passing update steps in our base-\DEHNN{} provides an universal approximation of the desired nested-permutation invariant update functions over nets. The proof of this theorem and more discussions can be found in the Supplement.

\subsection{Augmenting base-\DEHNN{} to \DEHNN}
\label{subsec:augmentation}

The base-\DEHNN{} provides an effective way to update features for both nodes and hyperedges in an iterative manner. We now describe further augmentation strategies for the resulting \DEHNN{} to capture long-range interactions as well as to be more senstive to the multi-scale graph topology. 

\paragraph{Hierarchy of virtual nodes.} 
The standard message-passing GNNs have difficulty to capture long-range interaction due to issues such as over-smoothing, over-squashing and under-reaching. 
Most GNNs used in practice have few layers. 
Transformer-type models for graphs may alleviate the issue \citep{NEURIPS2022_8c3c6668}, but each self-attention layer requires quadratic computation, which does not scale to large graphs. 
Furthermore, it is not clear how to effectively encode graph structures for transformers, even with the use of initial position encoding \citep{mueller2023attending} or reweighting based on graph distances \citep{zhang2023rethinking}. %
Intuitively, we want to keep the simple and efficient message-passing framework over the input (very sparse) hypergraph, but also have a way to propagate long-range information among nodes that are far away from each other (in terms of graph distances). We will do so via the use of \emph{virtual nodes}. 

Specifically, a \emph{virtual node (VN)} is an additional node we add that is connected to all input nodes. This effectively reduces the graph diameter to $2$, and has been a popular strategy in graph learning literature; e.g. \citep{10.5555/3305381.3305512}. However, note that messages cannot be directly passed between two nodes. Instead, information of all nodes have to be aggregated at the virtual node before being passed back to all nodes. Nevertheless, \citep{cai2023connection} shows that a message-passing GNN augmented by a single VN can already approximate Lineaer Transformer \citep{10.5555/3524938.3525416} and Performer \citep{choromanski2020rethinking} (as well as general transformers to some extent) and bring significant empirical gains.

\begin{figure}%
\centering
\includegraphics[width=0.30\textwidth]{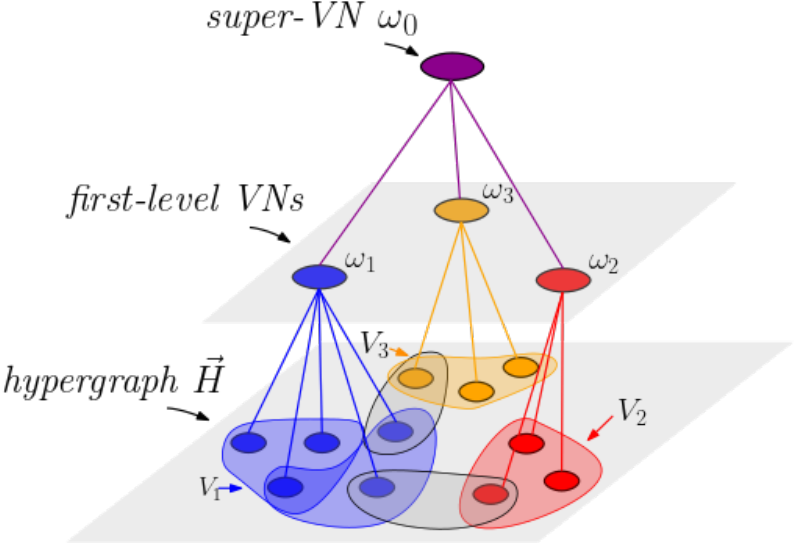}
\caption{A two-level hierarchy of virtual nodes (VNs).}
\label{fig:hvn}
\end{figure}

While adding a single VN will only linearly increase the number of edges, the VN itself will have a high degree and potentially become a computatoinal bottleneck. Furthermore, since the features of all nodes have to be aggregated at the virtual node, the aggregated messages will lose sensitivity to individual node features. The benefits of adding a single VN thus diminishes as the graph becomes larger. 
We instead propose to use a {\textbf{hierarchy of VNs}} as follows (see Figure~\ref{fig:hvn}). 

\begin{compactitem}
\item Given an input hypergraph $\adH = (V, \hE)$, we first use Metis \citep{doi:10.1137/S1064827595287997} to partition its node set to $k$ disjoint subsets $V = V_1 \cup V_2 \cup \cdots \cup V_k$. In our experiments, we keep the sizes of $V_i$s roughly balanced, and therefore the value of $k$ varies as input graph size changes. 

\item We introduce a VN $\avn_i$ for each subset $V_i$, for $i\in [1, k]$, and $\avn_i$ is connected to all nodes in $V_i$. These VNs are \emph{first-level VNs}. 
Note that the total number of edges added this way is only $n = |V|$. 

\item We further create a \emph{super-VN}, denoted by $\avn_0$, which connects to all the first-level VNs. This introduce an additional $k$ number of edges. 
\end{compactitem}

In what follows, we refer to nodes from input hypergraphs as \emph{standard nodes}. During the updating of node/edge features, we will use \emph{heterogeneous updates}, where the aggregation functions at the first-level VNs and super-VN are different from that at standard nodes. 
Different from graph pooling, the use of additional VNs keep the original hypergraph topology, while they act as additional ``bridges'' and allow information flow both at global (among far-apart nodes via VNs) and local scales (along original edges).

As the input hypergraph becomes even larger, we could use multiple layers in our hierarchy of VNs. Nevertheless, we observe that two layers already yield good performance in our current test cases. 
If nodes are placed, one could partition the node set by spatial locations (e.g., by decomposing the rectangular region) instead of using Metis.  

\paragraph{Initial positional / structural encodings.}
Finally, we aim to encode meaningful multi-scale features for each node / hyperedge. 
To this end, for each node, we add the Laplacian positional encoding consisting of the function value of this node from the first $s$ ($s=10$ in our experiments) eigenvectors of the undirected version of input hypergraph. %

Furthermore, to better capture the ``shape'' of the neighborhood of each node $v$ in a more discriminative and multi-scale manner, similar to \citep{yan2021link,zhao2020persistence}, we use the so-called extended persistence diagram (PD) summary induced by the shortest path distance function within the $r$-hop neighborhood of each node ($r=6$ in our experiments) as an initial \emph{structural encoding} for each node. Note that it is known \citep{TW19} that PDs constructed this way can encode rich graph information of the $r$-hop neighborhood of each node $v$ (viewed as a local motif around $v$), such as clustering coefficients, number of nodes at distance $a\le r$ to $v$, number of independent cycles, etc.
See the Supplement for more details. 
Indeed, as our ablation study and results in Supplement show, adding PDs improve \emph{both our and previous graph learning models}.

\section{Experiments}
\label{sec:exp}

\subsection{Datasets}

We used 12 of the Superblue circuits from \citep{ispd2011benchmarks,dac2012benchmarks} to evaluate our proposed models and baselines. The Superblue circuits are some of the most complex yet publicly available circuits used in previous work about VLSI placement and routing. The size of these netlists range from 400K to 1.3M nodes, with similar number of nets. Note that these hypergraphs are very sparse, although there are a very small fraction of nets with large sizes. See the Supplement for more detailed statistics of these netlists, including the size of each design as well as the statistics of target properties.

We pulled our designs from the RosettaStone \citep{9785815} repository's Skywater 130 nm technology \citep{edwards2020skywater} benchmark set.  We then used a commercial physical design tool to perform placement optimization for each design with a utilization ratio of 0.7.  We exported global-routing congestion information in the form of demand and capacity for each global-route cell (GRC) in each routing layer.

\subsection{Setup}

\paragraph{Baselines.} 
We compare with a range of SOTA baselines: 
\textbf{(i)} Graph Convolutional Networks (GCN) \citep{kipf2017semisupervised}; \textbf{(ii)} a SOTA variant of Graph Attention Networks (GATv2) \citep{DBLP:journals/corr/abs-2105-14491}; 
\textbf{(iii)} Hypergraphs Message Passing Neural Network (HMPNN) \citep{Heydari_2022}; \textbf{(iv)} Hypergraph (Neural) Networks with Hyperedge Neurons (HNHN) \citep{dong2020hnhn}; \textbf{(v)} A multiset function framework for Hypergraph Neural Network (AllSet) \citep{chien2022you}; \textbf{(v)} the SOTA graph-based model specifically defined for netlist property predictions, NetlistGNN \citep{NEURIPS2022_7fa54815}. See the Supplement for more details on their architecture, including model complexity.
We also compare with the hypergraph convolutional operator (HyperConv) \citep{DBLP:journals/corr/abs-1901-08150} and Linear Transformers \citep{10.5555/3524938.3525416}. As both models' performances on average are not as good as other baseline models, we include their results only in the Supplement.

We compare these baselines with our \DEHNN{} model, which we refer to as {\sf full-\DEHNN{}} in results to differentiate with base-\DEHNN{} model, whose performance we also include for reference. Note that base-\DEHNN{} is {\bf without} persistence diagrams (PDs) as initial features {\bf nor} virtual nodes (VNs).
We implement our \DEHNN{} and the baselines with PyTorch \citep{NEURIPS2019_bdbca288} and PyTorch-geometric \citep{fey2019fast}. 

\myparagraph{Features.}
For each cell, its initial features include: type, width, height of gate, degree, degree distribution of a local neighborhood (summarized into a vector), top-10 Laplacian eigenvectors, and persistence diagram (PD) features (see Supplement). 
For all methods other than Linear Transformer, the initial net features are the nets' degrees. For Linear Transformer, net features are obtained by averaging the features of those nodes contained in it so as to provide more topology information. 
(see Supplement). 
For a fair comparison, the same initial cell/net features are used for all models (other than our base-\DEHNN{}, which do not have PDs as it is a base model without any augmentation). One might wonder whether adding persistence diagrams (PDs) features are beneficial for both our method and baselines. Indeed as we show later in Ablation study and in Supplement, using PDs improve model performance for {\bf both our methods and baselines}. For example, it improves NetlistGNN by $3.6\%$ on average in terms of demand regression in the single-design setting. 

\myparagraph{Prediction tasks.}
We test three different tasks in our experiments, including both net- and node-based tasks. These tasks cover the types of experiments performed by previous netlist prediction approaches.  

\begin{compactitem}
\item \textbf{Net-based wirelength regression}: We use half-perimeter wirelength (HPWL), a common estimator used for wirelength calculation, to calculate the wirelength of each net \citep{Mirhoseini2021AGP}. Similar to \citep{NEURIPS2022_7fa54815}, we take the $\log_{2}$ of the wirelength to reduce the range. 
\item \textbf{Net-based demand regression}: We predict the net-based {\it demand}. Congestion happens when demands exceeds capacity. There is no concensus on how to define congestion (difference or ratio) and we thus directly predict demand. 
    \item \textbf{Cell-based congestion classification}: Similar to \citep{NEURIPS2022_7fa54815} and \citep{10.1145/3489517.3530675}, we classify the cell-based congestion values (computed as the ratio of demand/capacity) into (a) [0, 0.9], \emph{not-congested}; and (b) [0.9, $\inf$]; \emph{congested}. 
\end{compactitem}

Our experiments have two settings: 
\begin{compactitem}
    \item \textbf{Single-design:} We train and evaluate on each individual design separately. For each graph in a design, we apply 4-fold cross-validation to randomly split the nodes into 4 subsamples with same sizes (25\%/25\%/25\%/25\%). We report the average performance across all 12 designs with cross-validation applied to each design. The distribution of target values and the average performance from cross-validation for each design can be found in the Supplement. 
    \item \textbf{Cross-design:} We aim to evaluate the ability of the models to generalize to unseen netlist topologies. Following previous work \citep{NEURIPS2022_7fa54815}, we use 10 designs, Superblue1,2,3,5,6,7,9,11,14,16, for training, superblue18 for validation, and superblue19 for testing. 
\end{compactitem}

\subsection{Results}

 \begin{table}
\caption{{\small Net-based wirelength regression. Last row ``Improvement'' refers to the improvement of our full \DEHNN{} model over the best baseline approach for each metric.}}
\centering
\resizebox{1.0\columnwidth}{!}{
\begin{tabular}{c|ccc|ccc}
\toprule
\textbf{Model} & \multicolumn{3}{c|}{Single-Design} & \multicolumn{3}{c}{Cross-Design} \\
\cmidrule{2-7}
& \textbf{RMSE} $\downarrow$ & \textbf{MAE} $\downarrow$ & \textbf{Pearson} $\uparrow$ & \textbf{RMSE} $\downarrow$ & \textbf{MAE} $\downarrow$ & \textbf{Pearson} $\uparrow$ \\ 

\midrule
\large{GCN} & \Large{$1.762$} & \Large{$1.276$} & \Large{$0.750$} & \Large{\color{blue}$1.691$} & \Large{\color{blue}$1.276$} & \Large{\color{blue}$0.746$} \\ 
\large{GATv2} & \Large{$1.812$} & \Large{$1.330$} & \Large{$0.687$} & \Large{$1.717$} & \Large{$1.281$} & \Large{$0.737$} \\
\large{AllSet} & \Large{\color{blue}$1.718$} & \Large{\color{blue}$1.264$} & \Large{\color{blue}$0.760$} & \Large{$1.837$} & \Large{$1.348$} & \Large{$0.695$} \\
\large{HMPNN} & \Large{$1.841$} & \Large{$1.368$} & \Large{$0.710$} & \Large{$1.785$} & \Large{$1.335$} & \Large{$0.710$} \\
\large{HNHN} & \Large{$1.852$} & \Large{$1.368$} & \Large{$0.717$} & \Large{$1.754$} & \Large{$1.333$} & \Large{$0.701$} \\
\large{NetlistGNN} & \Large{$1.773$} & \Large{$1.320$} & \Large{$0.740$} & \Large{$1.762$} & \Large{$1.324$} & \Large{$0.718$} \\
\midrule
\large{base \DEHNN} & \Large{$1.751$} & \Large{$1.269$} & \Large{$0.748$} & \Large{$1.731$} & \Large{$1.291$} & \Large{$0.730$} \\
\large{\textbf{full \DEHNN}} & \Large{\color{red}$\bm{1.689}$} & \Large{\color{red}$\bm{1.245}$} & \Large{\color{red}$\bm{0.770}$} & \Large{\color{red}$\bm{1.677}$} & \Large{\color{red}$\bm{1.242}$} & \Large{\color{red}$\bm{0.754}$} \\
\midrule
\large{\textbf{Improvement}} & \Large{$\bm{1.7\%}$} & \Large{$\bm{1.6\%}$} & \Large{$\bm{1.3\%}$} & \Large{$\bm{1.9\%}$} & \Large{$\bm{2.6\%}$} & \Large{$\bm{1.8\%}$}\\
\bottomrule

\end{tabular}
}
\vspace{-0.5cm}
\label{table:hpwl}
\end{table}

 \begin{table}
\caption{Net-based demand regression for each design. }
\centering
\resizebox{1.0\columnwidth}{!}{
\begin{tabular}{c|ccc|ccc}
\toprule
\textbf{Model} & \multicolumn{3}{c|}{Single-Design} & \multicolumn{3}{c}{Cross-Design} \\
\cmidrule{2-7}
& \textbf{RMSE} $\downarrow$ & \textbf{MAE} $\downarrow$ & \textbf{Pearson} $\uparrow$ & \textbf{RMSE} $\downarrow$ & \textbf{MAE} $\downarrow$ & \textbf{Pearson} $\uparrow$ \\

\midrule
\large{GCN} & \Large{$9.321$} & \Large{$6.163$} & \Large{$0.570$} & \Large{$6.571$} & \Large{$5.024$} & \Large{$0.365$} \\
\large{GATv2} & \Large{$9.342$} & \Large{$6.118$} & \Large{$0.561$} & \Large{$6.623$} & \Large{$5.137$} & \Large{$0.363$} \\
\large{AllSet} & \Large{$9.072$} & \Large{\color{blue}$5.745$} & \Large{\color{blue}$0.632$} & \Large{\color{blue}$6.120$} & \Large{\color{blue}$4.820$} & \Large{$0.345$} \\
\large{HMPNN} & \Large{$9.342$} & \Large{$6.118$} & \Large{$0.561$} & \Large{$6.979$} & \Large{$5.356$} & \Large{$0.306$} \\
\large{HNHN} & \Large{$9.119$} & \Large{$5.885$} & \Large{$0.594$} & \Large{$6.390$} & \Large{$4.870$} & \Large{$0.358$} \\
\large{NetlistGNN} & \Large{\color{blue}$9.063$} & \Large{$5.839$} & \Large{$0.623$} & \Large{$8.328$} & \Large{$6.839$} & \Large{\color{blue}$0.367$} \\
\midrule
\large{base \textbf{\DEHNN}} & \Large{$8.997$} & \Large{$5.764$} & \Large{$0.630$} & \Large{$6.778$} & \Large{$5.085$} & \Large{$0.337$} \\
\large{\textbf{full \DEHNN}} & \Large{\color{red}$\bm{8.381}$} & \Large{\color{red}$\bm{5.334}$} & \Large{\color{red}$\bm{0.683}$} & \Large{\color{red}$\bm{6.037}$} & \Large{\color{red}$\bm{4.670}$} & \Large{\color{red}$\bm{0.372}$} \\
\midrule
\large{\textbf{Improvement}} & \Large{$\bm{7.5\%}$} & \Large{$\bm{7.2\%}$} & \Large{$\bm{8.1\%}$} & \Large{$\bm{1.4\%}$} & \Large{$\bm{4.1\%}$} & \Large{$\bm{1.4\%}$} \\
\bottomrule

\end{tabular}
}
\vspace{-0.5cm}
\label{table:demand}
\end{table}
 
\begin{table}
\caption{Cell-based congestion classification.}
\centering
\resizebox{1.0\columnwidth}{!}{
\begin{tabular}{c|ccc|ccc}
\toprule
\textbf{Model} & \multicolumn{3}{c|}{Single-Design} & \multicolumn{3}{c}{Cross-Design} \\
\cmidrule{2-7}
& \textbf{Precision} $\uparrow$ & \textbf{Recall} $\uparrow$ & \textbf{F\_score} $\uparrow$ & \textbf{Precision} $\uparrow$ & \textbf{Recall} $\uparrow$ & \textbf{F\_score} $\uparrow$  \\ 

\midrule
\large{GCN} & \Large{$0.761$} & \Large{$0.857$} & \Large{$0.802$} & \Large{$0.633$} & \Large{$0.997$} & \Large{\color{blue}$0.773$}\\
\large{GATv2} & \Large{$0.810$} & \Large{$0.864$} & \Large{\color{blue}$0.835$}  & \Large{$0.630$} & \Large{$\color{red}\bm{0.999}$} & \Large{$0.765$}\\
\large{AllSet} & \Large{$0.782$} & \Large{$0.837$} & \Large{$0.804$} & \Large{$0.645$} & \Large{$0.964$} & \Large{\color{blue}$0.773$} \\
\large{HMPNN} & \Large{$0.774$} & \Large{$0.826$} & \Large{$0.792$} & \Large{$0.633$} & \Large{$\color{red}\bm{0.999}$} & \Large{$0.772$} \\
\large{HNHN} & \Large{$0.792$} & \Large{\color{blue}$0.869$} & \Large{$0.826$} & \Large{\color{blue}$0.648$} & \Large{$0.939$} & \Large{$0.767$} \\
\large{NetlistGNN} & \Large{\color{blue}$0.812$} & \Large{$0.860$} & \Large{$0.831$} & \Large{$0.647$} & \Large{$0.953$} & \Large{$0.771$} \\
\midrule
\large{base \textbf{\DEHNN}} & \Large{$0.824$} & \Large{$0.860$} & \Large{$0.840$} & \Large{$0.653$} & \Large{$0.990$} & \Large{$0.774$} \\
\large{\textbf{full \DEHNN}}
& \Large{\color{red}$\bm{0.833}$} & \Large{\color{red}$\bm{0.876}$} & \Large{\color{red}$\bm{0.853}$} & \Large{\color{red}$\bm{0.660}$} & \Large{$0.986$} & \Large{\color{red}$\bm{0.780}$} \\
\midrule
\large{\textbf{Improvement}} & \Large{$\bm{2.6\%}$} & \Large{$\bm{0.8\%}$} & \Large{$\bm{2.2\%}$} & \Large{$\bm{1.7\%}$} & \Large{$\bm{-}$}  & \Large{$\bm{1.0\%}$}\\
\bottomrule

\end{tabular}
}
\vspace{-0.5cm}
\label{table:classify}
\end{table}

The test performance of all baselines and our methods (base-\DEHNN{} and full-\DEHNN) are shown in Tables~\ref{table:hpwl}, \ref{table:demand} and \ref{table:classify}. For the regression tasks, to be comprehensive, we use three metrics: the Mean Squared Error (MSE), the Mean Average Error (MAE), and the Pearson correlation (Pearson). For classification tasks, we report the Precision, Recall, and F-Score. Note that for MSE and MAE, the smaller the value is the better; while for Pearson, Precision, Recall, and F-Score, the larger the better.  
In all tables, we highlight the best performance results in {\bf red}, while the best among all baseline models are colored {\bf blue} (unless a baseline result is the best, in which case it will be red-colored). 
In the last row of these tables, we show the {\bf Improvement} of our full-\DEHNN{} model over the {\bf best of all baselines} for each metric; note that our improvement over any individual baseline can only be better than this improvement. 
Our full-\DEHNN{} model outperforms {\bf all} baselines, sometimes significantly. %
For example, compared to NetlistGNN \citep{NEURIPS2022_7fa54815}, the previous SOTA for netlist representatin learning, our improvement on average is around $5.3\%$ for wirelength prediction, and $8.6\%$ for demand prediction, both in terms of MAE. 
See the Supplement for the full results, including the test performances for each design in the single-design setting.     

\myparagraph{Ablation study.}
We carried out an ablation study to understand the effects of the different strategies employed in our \DEHNN{}.
In particular, the factors we wish to test are: (a) the use of direction in modeling nets, (b) the use of persistence diagrams (PDs) as features, and (c) the use of single and hierarchical VNs (virtual nodes). To this end, we compare the performance of the following versions: 
(a) {\sf base-E-HNN} stands for treating a net as a standard hyperedge (thus {\bf no direction}), and using {\bf neither} PD features {\bf nor} VNs. 
(b) {\sf base-\DEHNN{}} is the base model for directed hypergraph (described in Section~\ref{subsec:baseDEHNN}) with {\bf neither} PDs {\bf nor} VNs. In other words, the difference between base-\DEHNN{} and {\sf base-E-HNN} shows the effect of adding directions. 
(c) {\sf base-\DEHNN{}+PD} is the base model with only PDs. Hence the difference between (c) and (b) is to show the effect of PDs. 
(d) {\sf base-\DEHNN+PD+single VN} is the base model with PD and a single global VN. 
(e) Finally, {\sf full-\DEHNN{}} is our full model with PDs and a two-level hierarchy of VNs. 
The results for net-based demand regression and cell-based congestion classification are shown in Figure~\ref{fig:ablation}.
Full results are found in the Supplements, Other metrics and tasks show a similar behavior.
For example, for demand prediction, from {\sf base-E-HNN}, to adding directions, PDs, single VN, and finally two-level VNs, performance improves over the previous version by $2.5\%$, $2.6\%$, $1.0\%$ and $3.4\%$, respectively. Overall, full-\DEHNN{} improves over the {\sf base-DE-HNN} by around $6.8\%$, while its improvement over {\sf base-E-HNN} (base model with no direction) is $9.2\%$.  
\begin{figure}[htbp]
    \centering
    \includegraphics[width=0.48\textwidth]{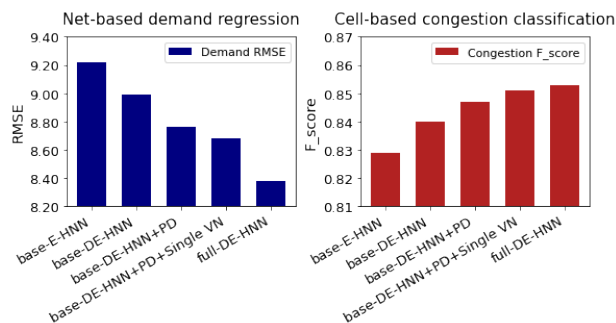}
    \caption{{\small Ablation study for net-based demand regression (left, RMSE) and cell-based congestion classification (right, F-score).}}
    \label{fig:ablation}
\end{figure}

\section{Conclusion} 
In this paper, we presented an effective model for representation learning on directed hypergraphs.
We considered learning on netlists, the hypergraph representation of circuits in chip design. This has great importance in practice but so far ML approaches for netlist representation learning has been limited, partly due to their huge sizes, long-range interactions, as well as the scarsity of benchmark datasets. 
We introduced several strategies to capture long-range interactions, graph motifs, and to consider direction in a hyperedge. Our model significantly outperforms a range of SOTA methods over a collection of chip designs.
Our datasets will be publicly available, which we hope will facilitate further research on ML for chip design applications, pushing the ability of methods to capture long-range interactions.
Finally, while we significantly outperformed other approaches, we note that in general, improvements in the cross-design setting are less prominent, potentially due to large variations in circuit designs. It will be interesting to further explore this direction.

\subsubsection*{Acknowledgements}
This work is partially supported by NSF under grants CCF-2112665 and  CCF-2310411, as well as by a gift fund from QualComm. The first, second and last authors would like to extend our sincere thanks to Prof. Andrew B. Kahng for the many insightful discussions and contributions in the early stage of this work, especially in exploring and experimenting the use of topological summaries for property prediction of netlists.

\bibliography{paper}
\bibliographystyle{apalike}

\appendix
\onecolumn
\aistatstitle{Supplement of the AISTATS submission \#1797, \\Title: ``DE-HNN: An effective neural model for Circuit Netlist representation"}
\section{More Background}
\begin{figure}[h]
    \centering
    \includegraphics[scale=0.7]{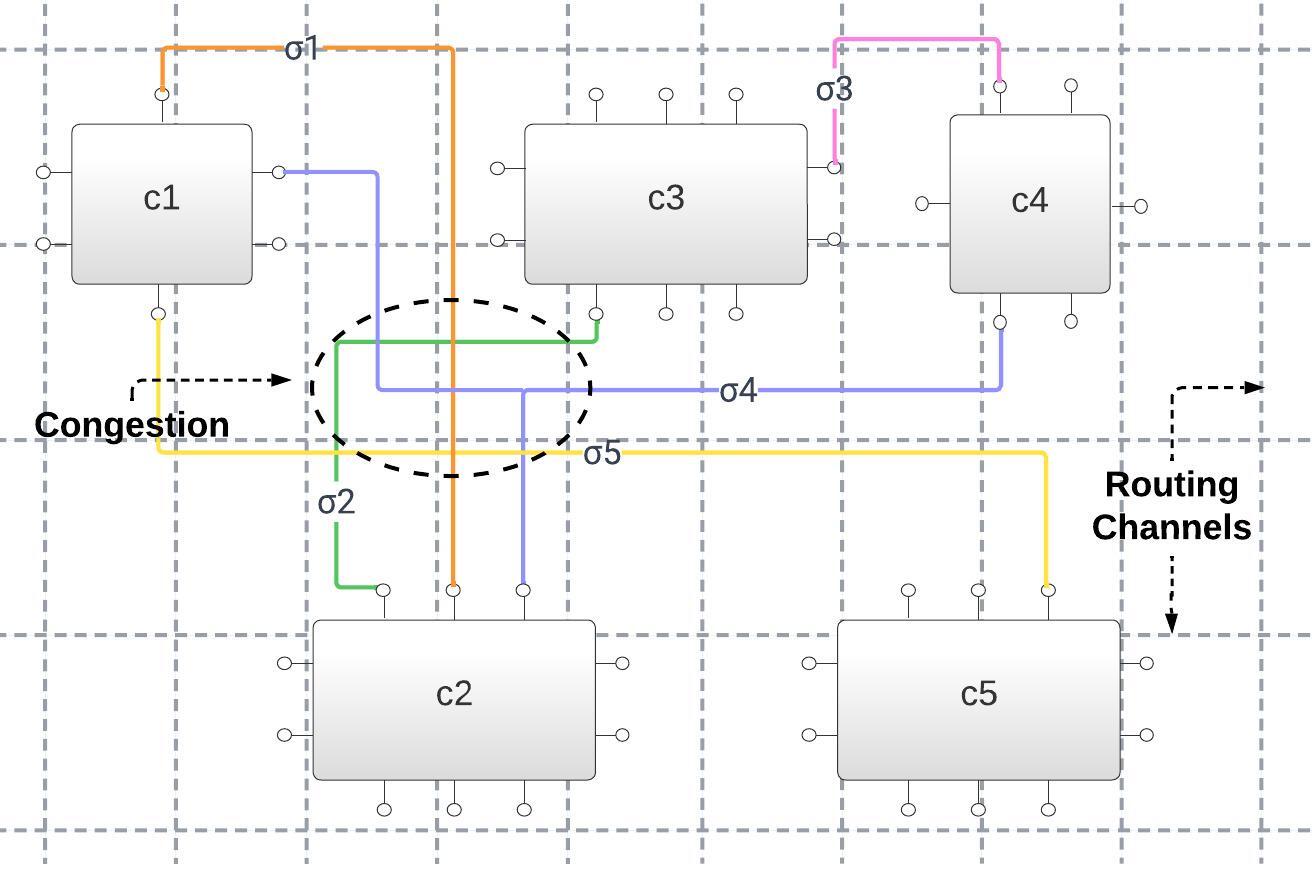}
    \caption{Visulization of a circuit netlits in the {\bf post place-and-route} stage: Each cell ($c_i$) is positioned within the physical layout of the chip and interconnected with other components following the net maps ($\sigma_i$).}
    \label{fig:circuit_diagram}
\end{figure}
\subsection{Circuit netlists}
A circuit netlist is a textual representation of electronic components, such as logical boolean gates, and the connections between them. Figure \ref{fig:circuit_diagram} provides an example illustrating a circuit netlist consisting of five components interconnected by five nets. After laying out the netlist in the physical space (placement), during the routing stage, the edges of the netlist are mapped to the routing channels within the chip's physical floorplan (indicated by dashed lines in Figure \ref{fig:circuit_diagram}). Routing congestion occurs when the number of edges to be routed in a specific region of the floorplan exceeds the available routing capacity.
\label{appendix:background}
\\ \\ \\ \\ \\ \\
\subsection{Persistence homology based features}
\label{appendix:PD}
Persistent homology is one of the most important development in the field of topological data analysis in the past two decades. It can encode meaningful topological features in a multi-scale manner and has already been applied to numerous applications; see books \citep{EH10, DW22}. 
Intuitively, given a domain $X$, a $p$-th homology class (or informally, a $p$-th homological feature) essentially captures a family of equivalent $p$-dimensional ``holes'' in $X$; for example, 0-, 1- and 2-dimensional homology captures connected components, equivalent classes of loops and of 2-dimensional voids, respectively. The $p$-th homology group $\mathrm{H}_p(X)$ (using $\mathbb{Z}_2$ coefficients) characterizes the space of $p$-th homological features of $X$, and its rank $rank(\mathrm{H}_p(X))$ corresponds to the number of independent $p$-th homological features. For example, $rank(\mathrm{H}_0(X))$ denotes the number of connected components of space $X$. If the domain $X$ is a graph, then its $1$-st homology group is homeomorphic the space of cycles (loops), and $rank(\mathrm{H}_1(X))$ is simply the number of independent cycles. 

Persistent homology is a modern extension of homology, where instead of a single space $X$, we now inspect a sequence of growing spaces $X_1 \subseteq X_2 \subseteq \cdots X_m = X$, called a \emph{filtration} \footnote{Note that persistent homology has been extended for much broader families of filtrations then the growing sequence we describe here; see \citep{EH10, DW22}.}. Intuitively, we can view this filtration as an evolution of the space $X$. As the space grows, we track its corresponding topological features (as captured by the homology groups introduced above). Sometimes new topological features (e.g., a new component, or a hole) will appear, and sometimes they will disappear (e.g, a hole is filled and thus ``disappear''). The \emph{persistent homology (PH)} captures such creation and death of topological features, and outputs a feature summary called \emph{persistent diagram (PD)}, consisting of the birth-time and death-time of classes of topological features during this evolution. In short, the PDs provide a multiscale summary for the topological features of $X$ from the perspective of filtration $X_1\subseteq X_2\subseteq \cdots \subseteq X_m$. In the past few years, there have been a series of approaches to vectorize PDs or to kernelize them for ML applications; see Chapter 13 of \citep{DW22}. In our work, we use the so-called \emph{persistence images} \citep{JMLR:v18:16-337} to convert a PD to a finite dimensional vector to be used as part of the node feature. 

\textbf{Persistence diagram (PD) induced from Netlists.}
We compute the persistence diagram (PD) based on the following directed graph representation of the netlist (instead of using the directed hypergraph representation, for easier computation of persistence). 

In particular, given a netlist, we create a directed graph $G = (V, E)$ where each node in $V$ corresponds to a cell in the netlist, while a directed edge is formed between the driver cell of some net with each of the sink in that net; that is, each net $\sigma = (\dr_\sigma, \sinkset_\sigma)$ gives rise to a set of directed edges $\{(\dr_\sigma, v)\}_{v\in \sinkset_\sigma}$. 
We refer to this graph as the \emph{star-graph} induced by the input netlist. 

Now, for each node $v$ in the star-graph $G$, we create its $k$-ring \emph{out-flow neighborhood $G^{in}_v$} which is simply the subgraph in $G$ spanned by all nodes reachable from $v$ by a path of at most $k$ hops. 
Symmetrically, the $k$-ring \emph{in-flow neighborhood $G^{out}_v$} of $v$ is the subgraph of $G$ spanned by all nodes such that $v$ is reachable within $k$ hops. These two $G_v$s form the \emph{local motifs} around $v$.
We wish to obtain a feature vector summary for $G^{in}_v$ and $G^{out}_v$. 
To this end, similar to \citep{yan2021link,zhao2020persistence}, we use the so-called \emph{$0$th extended persistence diagram} $PD^{*}_v$ \citep{cohen2009extending} of $G^*_v$ (where $* \in \{in, out\}$) induced by the shortest path distance function to $v$ as its summary. 
For our specific graph settig, using results of \citep{TW19}, one can show that $PD^*_v$ computed this way, is a concise summary that encodes rich information around $v$, including the number of triangles incident to $v$, clustering coefficient of $v$, number of nodes at distance $r \le k$ away from $v$, number of crossing edges from distance-$r$ nodes to distance ($r+1$) nodes, certain ``shortest" system of cycle basis passing through $v$ of $G^*_v$, and so on. 
Finally, each PD (for in-flow neighborhood and out-flow neighborhood of $v$) is vectorized to persistent images as in \citep{JMLR:v18:16-337}, and we further vectorize (flatten) each image matrix all together to generate a new vector as the persistent representation vector for each node. 

\subsection{Metis Partitioning}\label{appendix:metis}
Metis \citep{doi:10.1137/S1064827595287997} is currently State-of-the-art (SOTA) algorithm in chip design applications to provide balanced k-way partitions for large graphs efficiently. We choose Metis due to its practical performance, that it produces multiple clusters of balanced sizes, and the ease it is to control the size of the clusters. Besides Metis, hMetis \citep{10.1109/92.748202} developed on the base of Metis for hypergraph parititioning is an alternative choice. In our paper we used Metis in case we need partitioning for other graph based methods we compare our method with, and Metis will give us consistent partitioning. Nevertheless, we expect that hMetis can be used without much impact to the performance of our pipeline. 

\section{Theoretical Analysis} \label{sec:theoretical_analysis}
\subsection{Preliminaries}\label{sec:prelim}
Let $\mathbb{D}$ be a domain, such as $\mathbb{Q}, \mathbb{R}$, and $\mathbb{R}^d$. Consider the set function $f : 2^{\mathbb{D}} \rightarrow \mathrm{codom}(f)$ where $\mathbb{D}$ is a countable domain. Then, we have 
\begin{align}\label{eqn:sumdecomp}
    \forall X \subseteq \mathbb{D}: f(X) = \rho \circ \Phi(X), \ \ \Phi(X) = \sum_{x \in X} \phi(x),
\end{align}
where $\phi:\mathbb{D} \rightarrow \mathrm{codom}(\phi) \subset \mathbb{R}$, and $\rho: \mathrm{codom}(\phi) \rightarrow \mathrm{codom}(f)$. This is the so-called sum-decomposable representation of $f$ via $\mathbb{R}$ in terms of $(\phi, \rho)$ basis functions. We refer to ambient space of $\mathrm{codom}(\phi)$ (in Eqn~(\ref{eqn:sumdecomp}), it is $\mathbb{R}$) as the model’s latent space~\citep{NIPS2017_f22e4747}. There is an extended sum-decomposable representation of $f$ on multisets whose elements are drawn from an uncountable domain (e.g., $\mathbb{R}^d$) by restricting the size of input multisets. 
Recently, \citet{tabaghi2023universal} proposed an encoder $\phi: \mathbb{R}^d \rightarrow \mathrm{codom}(\phi) \subset \mathbb{R}^{2dN}$ such that for any continuous multiset function $f:\mathbb{X}_{N, \mathcal{B}} \rightarrow \mathrm{codom}(f)$ we have,
\[
    \forall X \in \mathbb{X}_{N,\mathcal{B}} : f(X) = \rho \circ \Phi(X),
\]
where $\mathbb{X}_{N,\mathcal{B}}$ is the set of multisets of size $N$ with elements drawn from $\mathcal{B}$ --- a compact subset of $\mathbb{R}^d$ --- and $\Phi(X) = \sum_{x \in X} \phi(x)$. The latent space dimension of this representation is $2dN$. As a result of this representation, $\Phi:\mathbb{X}_{N, \mathcal{B}} \rightarrow \mathrm{codom}(\Phi) $ is an injective map where $\mathrm{codom}(\Phi) \subset \mathbb{R}^{2dN}$. Furthermore, \citet{tabaghi2023universal} show that $\rho$ is continuous over the latent space $\mathbb{R}^{2dN}$.

At their core, sum-decomposable models rely on a bijection between multisets and encoded features, that is, $X = \alpha \circ \Phi(X)$ for any multiset $X \in \mathbb{X}_{N,\mathcal{B}}$ and a bijective map $\alpha$. In the following proposition, we first generalize this result to multisets with varied sizes.
\begin{proposition} \label{prop:multivariate_continuous_bijection}
   Let $\mathcal{B}$ be a connected compact subset of $\mathbb{R}^d$. There exists a function $\phi: \mathbb{R}^d \rightarrow \mathbb{R}^{2dN}$ such that 
    \[
        \forall X \in \mathbb{X}_{\leq N, \mathcal{B}}: X = \alpha \bigg( \sum_{x \in X} \phi(x) \bigg) = \alpha \circ \Phi(X)
    \]
    where $\mathbb{X}_{\leq N, \mathcal{B}}$ is all multisets of size $\leq N$ with elements from $\mathcal{B}$, $\alpha$ is a continuous map over $\mathbb{R}^{2dN}$.
\end{proposition} 

\begin{proof}
As mentioned earlier, \citet{tabaghi2023universal} proposed a  continuous encoding function $\phi^{\prime}: \mathbb{R}^d \rightarrow \mathbb{R}^{ 2dN}$ such that $\Phi^{\prime}(x) = \sum_{x \in X} \phi^{\prime}(x)$ is an injective map over multisets with exactly $N$ elements, that is, $(\Phi^{\prime})^{-1} \circ \Phi^{\prime}(X) = X$ where elements of multiset $X$ is drawn from $\mathbb{R}^d$ and $|X| = N$ \footnote{Note that we trivially changed the notation from $\phi$ to $\phi^{\prime}$.}. 

\begin{lemma} \label{lem:phi_inv_cont}
    The function $\Phi^{\prime}$ is a homeomorphism.
\end{lemma}
\begin{proof}
 The function $\Phi^{\prime}$ is continous and injective by contruction. We want to show that $(\Phi^{\prime})^{-1}$ is continuous over $\mathrm{codom}(\Phi^{\prime}) = \{ \Phi^{\prime}(X): X \in \mathbb{X}_{N, \mathcal{B}}\}$ where $\mathcal{B}$ is a compact and connected subset of $\mathbb{R}^d$. The set $\mathbb{X}_{N, \mathcal{B}}$ is compact; refer to Lemma 5 in \citep{tabaghi2023universal}. Also, $\mathrm{codom}(\Phi^{\prime}) \subseteq \mathbb{R}^{2dN}$ forms a metric space with $\ell_2$ metric; Hence, it is also a Hausdorff space. By the inverse function theorem, $\Phi^{\prime}$ --- a continuous bijection from a compact space to a Hausdorff space --- has a continuous inverse; see Proposition 13.26 in \citep{sutherland2009introduction}.
\end{proof}

To extend the result of \Cref{lem:phi_inv_cont} to multisets of variable sizes, we follow the same proof sketch as the one used for the one-dimensional case \citep{wagstaff2019limitations}. In particular, let $x_{\circ} \in \mathbb{R}^{d} \setminus \mathcal{B}$ where $\inf_{x \in \mathcal{B}}\| x - x_\circ \|_2 > 0$. Then, we define $\phi(x) = \phi^{\prime}(x) - \phi^{\prime}(x_{\circ})$. For any multiset with $M \leq N$ elements from $\mathcal{B}$, say $X$, we have
    \begin{align*}
        \forall X \in \mathbb{X}_{\leq N, \mathcal{B}}: \Phi(X) &= \sum_{x \in X} \phi(x) = \sum_{x \in X}  \phi^{\prime}(x) - M \phi^{\prime}(x_{\circ} ) \\
        &= \Phi^{\prime}( X \cup \{ \{ \underbrace{x_{\circ} , \ldots, x_{\circ} }_{N-M} \} \} )+\mathrm{const},
    \end{align*}
where $\mathrm{const} = -N \phi^{\prime}(x_\circ) $. 
This is in the sum-decomposable form. Since $\Phi^{\prime}$ is an injective map over $\mathbb{X}_{N, \mathcal{B}}$, $\Phi$ is also an injective map over $\mathbb{X}_{\leq N, \mathcal{B}}$. Specifically, we can derive a closed-form expression for its inverse as follows:
    \begin{align*}
         \Big( {\Phi^{\prime}}^{-1} \circ (\Phi (X)  - \mathrm{const} \big) \Big) \cap \mathcal{B}  &= \Big( {\Phi^{\prime}}^{-1}  \circ  \Phi^{\prime} ( X \cup \{ \{ \underbrace{x_{\circ} , \ldots, x_{\circ} }_{N-M} \} \} )  \Big) \cap \mathcal{B}   \\
         &= ( X \cup \{ \{ \underbrace{x_{\circ} , \ldots, x_{\circ} }_{N-M} \} \} )  \cap \mathcal{B} \\
         &= X. 
    \end{align*}
In other words, we have $\Phi^{-1}(U) = {\Phi^{\prime}}^{-1} \big( U  - \mathrm{const} \big) \cap \mathcal{B}$ for all $U \in \mathrm{codom}(\Phi) = \{ \Phi(X): X \in \mathbb{X}_{\leq N, \mathcal{B}} \}$.

The function $\Phi: \mathbb{X}_{\leq N, \mathcal{B}} \rightarrow \mathrm{codom}(\Phi)$ is a continuous bijection. Its domain is compact because it is a finite union of comapct spaces, that is, $ \mathbb{X}_{\leq N, \mathcal{B}} = \cup_{n \in [N]}  \mathbb{X}_{n, \mathcal{B}}$. Furthermore, its codomain forms a metric space with $\ell_2$ distance; Hence, it is also a Hausdorff space. Therefore, by an inverse function theorem, $\Phi$ has a continuous inverse; \citep{sutherland2009introduction}.
If we let $\alpha = {\Phi}^{-1}$, we arrive at the Proposition's statement.
\end{proof}
\begin{remark}
    \Cref{prop:multivariate_continuous_bijection} gurantees the continuity of $\alpha$ over a compact set $\mathrm{codom}(\Phi) \subset \mathbb{R}^{2dN}$. This function can be continuously extended to the ambient space $\mathbb{R}^{2dN}$; refer to the continuous extension theorem  \citep{fitzpatrick1989klaus}.
\end{remark}
\subsection{Proof of \Cref{thm:nested_permutation_invariant}}
As discussed in the main text, the general net-value function is nested-permutation invariant and take the form of Eqn (\ref{eq:M}) in the main text. In what follows, we provide the detailed version of Theorem \ref{thm:nested_permutation_invariant} and show that such nested-permutation invariant function adapts a nested sum-decomposition aligned with our \DEHNN's node and net update rules.

\begin{restatable*}[Detailed]{theorem}{goldbach}
Let $\Fcal$ be any continuous, nested-permutation invariant, net-value function as in Eqn (\ref{eq:M}). For simplicity, assume both input nets and output of $M$ take values in a compact set $\mathcal{B} \subset \mathbb{R}^d$, 
    a connected compact subset of $\mathbb{R}^d$. We also let $N_v  = \max \{ | \NN(v )|:\forall v  \}$ and  $N_{\ahyperE} = \max \{ |\sinkset(\ahyperE)| : \forall \ahyperE  \} $.  
    Then we have that $\Fcal$ can be expressed as the following sum-decomposition:
    \begin{align*}
     \forall \sigma: \Fcal\Big( \{ \mu_{\anet'} \}_{\anet' \in \NN(\dr_\ahyperE)} , 
\Big\{ \{  \mu_{\anet'} \}_{\anet' \in \NN(v^{\prime})}  \Big\}_{v^{\prime} \in \sinkset_\ahyperE}  \Big) = \rho \Big( \sum_{ \sigma^{\prime} \in \NN(\dr_\ahyperE) } \phi_1 (\mu_{\sigma^{\prime}}),
        \sum_{v^{\prime} \in \sinkset_\ahyperE}\phi_2\big( \sum_{ \sigma^{\prime} \in \NN(v^{\prime}) } \phi_1 (\mu_{\sigma^{\prime}}) \big) \Big),
    \end{align*} 
    where $\phi_1: \mathbb{R}^d \rightarrow  \mathbb{R}^{d^{\prime}(d,N_v)}$, $\phi_2: \mathbb{R}^{d^{\prime}(d,N_v)} \rightarrow \mathbb{R}^{d^{\prime \prime}(d,N_v,N_{\ahyperE})}$, and $\rho: \mathbb{R}^{d^{\prime}(d,N_v)} \times \mathbb{R}^{d^{\prime \prime}(d,N_v, N_{\sigma})} \rightarrow \mathbb{R}^d$ are continuous functions.
\end{restatable*}

\begin{proof}
The general net-value function takes the following form:
\begin{align*}
    \forall \sigma: ~& \Fcal\Big( \{ \mu_{\anet'} \}_{\anet' \in \NN(\dr_\ahyperE)} , 
\Big\{ \{  \mu_{\anet'} \}_{\anet' \in \NN(v^{\prime})}  \Big\}_{v^{\prime} \in \sinkset_\ahyperE}  \Big) \nonumber.
\end{align*}
We assume $\Fcal$ is a continuous function that takes values in $\mathcal{B}$, a connected compact subset of $\mathbb{R}^d$, and $N_v  = \max \{ | \NN(v)|:\forall v  \}$. From Proposition \ref{prop:multivariate_continuous_bijection}, there exist continuous functions $\phi_1$ and $\alpha_1$ such that
\[
    \forall X \in \mathbb{X}_{\leq N_v, \mathcal{B}}  : X = \alpha_1 ( \sum_{x \in X} \phi_1(x) ),
\]
where $\mathrm{codom}(\phi_1) \subseteq \mathbb{R}^{2dN_v}$, that is, the latent dimension depends on $d$ and $N_v$ which we denote as $d^{'}(d,N_v)$. Therefore, we apply the result in Proposition \ref{prop:multivariate_continuous_bijection} to sets of net values (of size at most $N_v$) and express $\Fcal$ follows:
\begin{align*}
    \forall \ahyperE: \Fcal\Big( \{ \mu_{\anet'} \}_{\anet' \in \NN(\dr_\ahyperE)} , 
\Big\{ \{  \mu_{\anet'} \}_{\anet' \in \NN(v^{\prime})}  \Big\}_{v^{\prime} \in \sinkset_\ahyperE}  \Big) &=  \mathcal{F} \Big(  \alpha_1 \big( \sum_{\sigma^{\prime} \in \NN(\dr_\ahyperE)} \phi_1 (\mu_\sigma^{\prime}) \big) , 
    \Big\{ \alpha_1 \big( \sum_{\sigma^{\prime} \in \NN(v^{\prime})} \phi_1 (\sigma^{\prime}) \big)   \Big\}_{v^{\prime} \in \sinkset_\ahyperE}  
    \Big) \\
    &=  \mathcal{F}_1 \Big(  \sum_{\sigma^{\prime} \in \NN(\dr_\ahyperE)} \phi_1 (\mu_\sigma^{\prime}), 
    \Big\{  \sum_{\sigma^{\prime} \in \NN(v^{\prime})} \phi_1 (\sigma^{\prime}) \Big\}_{v^{\prime} \in \sinkset_\ahyperE}  \Big)
\end{align*}
where $\mathcal{F}_1(x, Y) \stackrel{\mathrm{def.}}{=} \mathcal{F}(\alpha_1(x) , \{ \alpha_1(y) \}_{y \in Y})$ for all $x \in \mathrm{codom}(\Phi_1) = \{ \Phi_1(Y) = \sum_{y \in Y} \phi_1(y) : Y \in \mathbb{X}_{\leq N_v, \mathcal{B}} \}$  and any set $Y$ with elements in $\mathrm{codom}(\Phi_1)$. Since both $\alpha_1$ and $\mathcal{F}$ are continuous functions, then $\mathcal{F}_1$ is also a continuous function. \footnote{We can use the matching distance to define a metric on multisets; refer to \citep{tabaghi2023universal} for details.}

Next, we can apply the result of Proposition \ref{prop:multivariate_continuous_bijection} to $\mathcal{F}_1$. The elements of the set $\Big\{ \sum_{\sigma^{\prime} \in \NN(v^{\prime})} \phi_1 (\sigma^{\prime})  \Big\}_{v^{\prime} \in \sinkset_\ahyperE}$ take values in $\mathrm{codom}(\Phi_1)$ ---  a connected compact subset of $\mathbb{R}^{d^{\prime}(d, N_v)}$. If we assume $|\sinkset_\ahyperE| \leq N_{\ahyperE}$ for all $\ahyperE$, Proposition \ref{prop:multivariate_continuous_bijection} claims that there exist continuous functions $\alpha_2$ and $\phi_2$ such that the following holds true:
\begin{align*}
    \forall \ahyperE: \mathcal{F}_1 \Big(  \sum_{\sigma^{\prime} \in \NN(\dr_\ahyperE)} \phi_1 (\mu_\sigma^{\prime}), 
    \Big\{  \sum_{\sigma^{\prime} \in \NN(v^{\prime})} \phi_1 (\sigma^{\prime}) \Big\}_{v^{\prime} \in \sinkset_\ahyperE}  \Big)
    &= \mathcal{F}_1 \Bigg(  \sum_{\sigma^{\prime} \in \NN(\dr_\ahyperE)} \phi_1 (\mu_\sigma^{\prime}), 
    \alpha_2 \Big( \sum_{v^{\prime} \in \sinkset_\ahyperE} \phi_2 \big(   \sum_{\sigma^{\prime} \in \NN(v^{\prime})} \phi_1 (\sigma^{\prime}) \big)   \Big) \Bigg) \\
    &= \rho \Bigg(  \sum_{\sigma^{\prime} \in \NN(\dr_\ahyperE)} \phi_1 (\mu_\sigma^{\prime}), 
     \sum_{v^{\prime} \in \sinkset_\ahyperE} \phi_2 \big(   \sum_{\sigma^{\prime} \in \NN(v^{\prime})} \phi_1 (\sigma^{\prime}) \big)    \Bigg)
\end{align*}
where  $\rho(x, y) \stackrel{\text{def}}{=} \mathcal{F}_1(x , \alpha_2(y) )$ for all $x \in \mathrm{codom}(\Phi_1)$ and $y \in \mathrm{codom}(\Phi_2) \stackrel{\text{def}}{=}   \{ \sum_{z \in Z}\phi_2(z): Z \in \mathrm{codom}(\Phi_1), |Z| \leq N_{\ahyperE} \} $. The function $\rho$ is a composition of continuous functions $\mathcal{F}_1$ and $\alpha_2$. Therefore, it is a continuous function. Also, we have $\mathrm{codom}(\phi_2) \subseteq  \mathbb{R}^{2d^{\prime}(d, N_v) N_{\ahyperE}}$, that is, the latent dimension depends on $d, N_v,$ and $N_\ahyperE$ which we denote as $d^{\prime \prime}(d, N_v,N_\ahyperE)$. Finally, functions $\phi_2$ and $\rho$ are continuous over compact domains $\mathrm{codom}(\Phi_1)$, and $\mathrm{codom}(\Phi_1) \times \mathrm{codom}(\Phi_2)$. Therefore, they can be continuously extended to $\mathbb{R}^{d^{\prime}(d,N_v)}$ and $\mathbb{R}^{d^{\prime}(d,N_v)} \times \mathbb{R}^{d^{\prime \prime}(d,N_v, N_{\sigma})}$.
\end{proof}

\section{Experimental details}

\subsection{Dataset Statistics}
\label{appendix:statistics}

We report the sizes of each design and their cell/net-degrees distribution in Table~\ref{table:dataset}, and net-based demand/wirelength distributions in the Table~\ref{table:distribution}. We can see that (hyper)graphs in our dataset are large and sparse, ranging from 400K to 1.3M nodes. 
with few outliers that have high degrees. We also summarize the more detailed distributions of net-based demand, net-based wirelength, and cell-based congestion in figure\ref{fig:statistics_hist}. 

\subsection{Experiment Setup}
\label{appendix:setup}

We engineer the input features as Cell features and Net features. None of the following features contained placement information or are computed from placement information. We engineer cell features as follows:
\begin{itemize}
\item We analyze the statistics in a design including the minimum and maximum of all cells' width, height. Then, we normalize all these quantities to be in the range $[0, 1]$. We concatenate them with the cell's discrete orientation and cell's degree (number of nets each cell connecting to), that results into the cell's input feature 
vector.
\item We calculate the top-10 eigenvectors of the graph Laplacian of the heterogenous graph as the positional encoding (denoted as LapPE) for every cell and net.
\item We compute the persistence diagram (denoted as PD) for each cell-node $v$ on a directed graph $G = (V, E)$, as we mentioned earlier in \ref{appendix:PD}, as a common cell-node input feature vector for all baselines, to capture the local topological information. 
\item We also compute the degree distribution for each cell-node $v$ based on their $k$-ring \textbf{\emph{undirected}} \emph{neighborhood $G_v$} in the star-graph we introduced in \ref{appendix:PD} to compute the persistence diagrams. However, we ignore the direction when we compute the degree distribution. 

\end{itemize}
For the net features, for all models other than linear Transformer, we initialize the features of each {\bf net} $\ahyperE = (\dr_\ahyperE, \sinkset_\ahyperE)$ as the net's degree (i.e.~number of cells each net connecting to). 
For Linear Transformer, in order to provide better topology information, we instead initialize the features of each {\bf net} $\ahyperE = (\dr_\ahyperE, \sinkset_\ahyperE)$ (denoted as $M(\ahyperE)$) as the average of all the features of the nodes this net contains, computed as, 
\begin{equation}\label{eq:lin_net}
M(\ahyperE) = \frac{1}{|\sinkset_\ahyperE| + 1} ( m(\dr_\sigma) +  \sum_{v^\prime \in \sinkset_\ahyperE} m(v^\prime) ). 
\end{equation}

As also described in the main text, for GCN and GATv2, we use the bipartite graph representation of the input netlist, where there are two types of graph nodes: those corresponding to cells and those corresponding nets. If a cell is contained in a net, then there is an edge between their respective graph nodes. We use {\bf heterogenous} message passing, where nodes corresponding to cells use a different message passing mechanism from nodes corresponding to nets. 

For GCN, GATv2, HyperConv, and all other HNN based models, we used 4 layers with 64 dimension each layer as the setting. For Linear Transformer, we used 2 layers with 64 dimension each layer as the setting. For NetlistGNN we used 4 layers with node dimension 64 and net dimension 64 as the setting. For E-HNN/DE-HNN based models, we used 3 layers with node dimension 64. 

We report the number of parameters, training time per epoch and total training epochs in the Table \ref{table:params}. 
We aim to use similar number of parameters, although currently full-\DEHNN does have the highest number of parameters, with NetlistGNN having the second largest. Interestingly, while they all take similar number of epochs to converge, our \DEHNN{} in fact is much faster per epoch, so overall takes less time to train.
We used two NVIDIA A100-SXM4-80GB GPUs on Linux CentOS Stream system (Ver.8) for both single-design and cross-design experiments. \par

\begin{table}[htbp]
\centering
\resizebox{1.0\columnwidth}{!}{
\begin{tabular}{c|cccccccccc}
\toprule
\textbf{Model} & GCN & GATv2 & HyperConv & Lin. Transformer & NetlistGNN & AllSet & HMPNN & HNHN & base-DE-HNN & full-DE-HNN \\ 
\midrule
\textbf{Num. of parameters} & 218113 & 251905 & 218625 & 203201 & 364743 & 348289 & 235137 & 217729 & 272165 & 383105 \\
\textbf{Time per epoch (single design)} & 2.77 & 2.91 & 1.06 & 3.72 & 1.13 & 0.79 & 0.69 & 0.62 & 0.49 & 0.65\\
\textbf{Total num. of epochs (single design)} & 878 & 775 & 670 & 636 & 716 & 689 & 691 & 699 & 673 & 810\\
\textbf{Time per epoch (cross design)} & 33.20 & 36.28 & 12.28 & 55.37 & 16.51 & 11.35 & 4.14 & 5.37 & 5.26 & 5.50\\
\textbf{Total num. of epochs (cross design)} & 478 & 482 & 473 & 471 & 455 & 460 & 472 & 468 & 479 & 473\\
\bottomrule

\end{tabular}
}
\caption{Comparison of Model complexity for all the models we used. We compare the number of parameters, training time per epoch for single design (on average), total training epoches to converge for single design (on average), trainig time per epoch for cross design, and total training epoches to converge for cross design.  } 
\vspace{-0.5cm}
\label{table:params}
\end{table}

\begin{table}[htbp]
\centering
\begin{tabular}{ccccccccccc}
\toprule
\multirow{1}{*}{\textbf{Design}} & \multirow{1}{*}{\textbf{Cells}} & \multirow{1}{*}{\textbf{Nets}} & \multicolumn{4}{c}{\textbf{Cell-degree}} & \multicolumn{4}{c}{\textbf{Net-degree}} \\
\cmidrule{4-7} \cmidrule{8-11}
& & & \textbf{Min} & \textbf{Max} & \textbf{Mean} & \textbf{STD} & \textbf{Min} & \textbf{Max} & \textbf{Mean} & \textbf{STD}\\ 

\midrule

\multirow{1}{*}{Superblue1} & \multirow{1}{*}{797,938} & \multirow{1}{*}{821,523}
& 0.0 & 1243 & 3.70 & 5.66 & 0.0 & 140605 & 3.58 & 155.85 \\

\midrule

\multirow{1}{*}{Superblue2} & \multirow{1}{*}{951,166} & \multirow{1}{*}{985,117} 
& 0.0 & 1317 & 3.51 & 5.66 & 0.0 & 190487 & 3.39 & 192.22 \\

\midrule

\multirow{1}{*}{Superblue3} & \multirow{1}{*}{901,254} & \multirow{1}{*}{925,667} 
& 0.0 & 2245 & 3.65 & 5.23 & 0.0 & 168630 & 3.55 & 175.91 \\
\midrule

\multirow{1}{*}{Superblue5} & \multirow{1}{*}{727,341} & \multirow{1}{*}{803,681} 
& 0.0 & 1381 & 3.46 & 5.82 & 0.0 & 114259 & 3.13 & 128.03 \\

\midrule

\multirow{1}{*}{Superblue6} & \multirow{1}{*}{998,122} & \multirow{1}{*}{1049,225} 
& 0.0 & 1689 & 3.57 & 4.03 & 0.0 & 179410 & 3.39 & 175.69\\

\midrule

\multirow{1}{*}{Superblue7} & \multirow{1}{*}{1319,052} & \multirow{1}{*}{1339,522} 
& 0.0 & 849 & 3.91 & 3.03 & 0.0 & 265765 & 3.85 & 230.24 \\

\midrule

\multirow{1}{*}{Superblue9} & \multirow{1}{*}{810,812} & \multirow{1}{*}{830.308} 
& 0.0 & 1265 & 3.83 & 4.78 & 0.0 & 129541 & 3.74 & 202.42 \\

\midrule

\multirow{1}{*}{Superblue11} & \multirow{1}{*}{923,355} & \multirow{1}{*}{954,144} 
& 0.0 & 1983 & 3.74 & 6.97 & 0.0 & 203194 & 3.63 & 294.36\\

\midrule

\multirow{1}{*}{Superblue14} & \multirow{1}{*}{604,921} & \multirow{1}{*}{627,036} 
& 0.0 & 1023 & 3.90 & 4.66 & 0.0 & 167911 & 3.76 & 300.30 \\

\midrule

\multirow{1}{*}{Superblue16} & \multirow{1}{*}{671,284} & \multirow{1}{*}{696,983} 
& 0.0 & 1016 & 3.77 & 6.12 & 0.0 & 140741 & 3.63 & 238.58 \\

\midrule

\multirow{1}{*}{Superblue18} & \multirow{1}{*}{459,495} & \multirow{1}{*}{468,888} 
& 0.0 & 1192 & 4.22 & 4.30 & 0.0 & 102047 & 4.14 & 150.74 \\

\midrule

\multirow{1}{*}{Superblue19} & \multirow{1}{*}{495,234} & \multirow{1}{*}{510,258} 
& 0.0 & 1507 & 3.58 & 6.02 & 0.0 & 94682 & 3.48 & 135.31 \\

\bottomrule

\end{tabular}
\caption{Dataset details \& statistics. 1st column in the table shows the name of the design, 2nd column and 3rd column show the number of cells and nets in each design, 4th-7th columns show the distribution of cell-degrees for each design and 8th-11th columns show the distribution of net-degrees for each design.}
\label{table:dataset}
\end{table}

\begin{table}[htbp]
\centering
\resizebox{1.0\columnwidth}{!}{
\begin{tabular}{ccccccccccccc}
\toprule
\multirow{1}{*}{\textbf{Design}} & \multicolumn{4}{c}{\textbf{Net-based demand}} & \multicolumn{4}{c}{\textbf{Net-based wirelength}} & \multicolumn{4}{c}{\textbf{Cell-based congestion}} \\
\cmidrule{2-5} \cmidrule{6-9} \cmidrule{10-13}
& \textbf{Min} & \textbf{Max} & \textbf{Mean} & \textbf{STD} & \textbf{Min} & \textbf{Max} & \textbf{Mean} & \textbf{STD} & \textbf{Min} & \textbf{Max} & \textbf{Mean} & \textbf{STD}\\ 

\midrule

\multirow{1}{*}{Superblue1} 
& 0.0 & 103.50 & 26.24 & 8.01 & 8.91 & 23.92 & 14.80 & 2.45 & 0.0 & 12.00 & 1.21 & 0.55\\

\midrule

\multirow{1}{*}{Superblue2}
& 0.0 & 139.37 & 26.61 & 8.52 & 8.91 & 24.61 & 15.21 & 2.69 & 0.0 & 5.00 & 0.70 & 0.45\\

\midrule

\multirow{1}{*}{Superblue3} 
& 0.0 & 103.50 & 24.07 & 7.31 & 8.91 & 23.83 & 14.91 & 2.45 & 0.0 & 5.19 & 0.90 & 0.47\\
\midrule

\multirow{1}{*}{Superblue5} 
& 0.0 & 1203.25 & 41.62 & 32.00 & 8.91 & 24.06 & 15.68 & 2.96 & 0.0 & 78.92 & 1.27 & 0.94\\

\midrule

\multirow{1}{*}{Superblue6} 
& 0.0 & 980.33 & 33.83 & 27.45 & 8.91 & 23.83 & 14.88 & 2.57 & 0.0 & 74.15 & 1.35 & 1.15\\

\midrule

\multirow{1}{*}{Superblue7}  
& 0.0 & 495.25 & 23.89 & 5.34 & 8.91 & 23.93 & 14.89 & 2.48 & 0.0 & 43.38 & 1.09 & 0.33\\

\midrule

\multirow{1}{*}{Superblue9}
& 0.0 & 79.50 & 22.47 & 8.02 & 8.91 & 23.83 & 14.71 & 2.42 & 0.0 & 5.0 & 0.74 & 0.46\\

\midrule

\multirow{1}{*}{Superblue11} 
& 0.0 & 75.00 & 20.05 & 6.54 & 8.91 & 24.17 & 15.11 & 2.38 & 0.0 & 8.0 & 0.93 & 0.36\\

\midrule

\multirow{1}{*}{Superblue14}
& 0.0 & 401.41 & 23.42 & 9.11 & 8.91 & 23.67 & 15.06 & 2.48 & 0.0 & 46.85 & 1.06 & 0.65\\

\midrule

\multirow{1}{*}{Superblue16} 
& 0.0 & 1091.0 & 28.96 & 14.09 & 8.91 & 23.51 & 15.01 & 2.54 & 0.0 & 65.53 & 1.29 & 0.91\\

\midrule

\multirow{1}{*}{Superblue18} 
& 0.0 & 50.0 & 20.39 & 4.13 & 8.91 & 23.26 & 14.98 & 2.28 & 0.0 & 4.0 & 0.91 & 0.27\\

\midrule

\multirow{1}{*}{Superblue19} 
& 0.0 & 87.83 & 23.05 & 6.40 & 8.91 & 23.53 & 14.86 & 2.36 & 0.0 & 6.0 & 0.96 & 0.50\\

\bottomrule

\end{tabular}
}
\caption{Dataset details \& statistics. This table shows the distribution of net-based demands for each design, distribution of net-based logged wirelength for each design and the distribution of cell-based congestion values for each design. Remind that cell-based congestion, as we described in the paper, we classify the cell-based congestion values (computed as the ratio of demand/capacity) into (a) [0, 0.9], \emph{not-congested}; and (b) [0.9, $\inf$]; \emph{congested}.}
\label{table:distribution}
\end{table}

\begin{figure}[htbp]
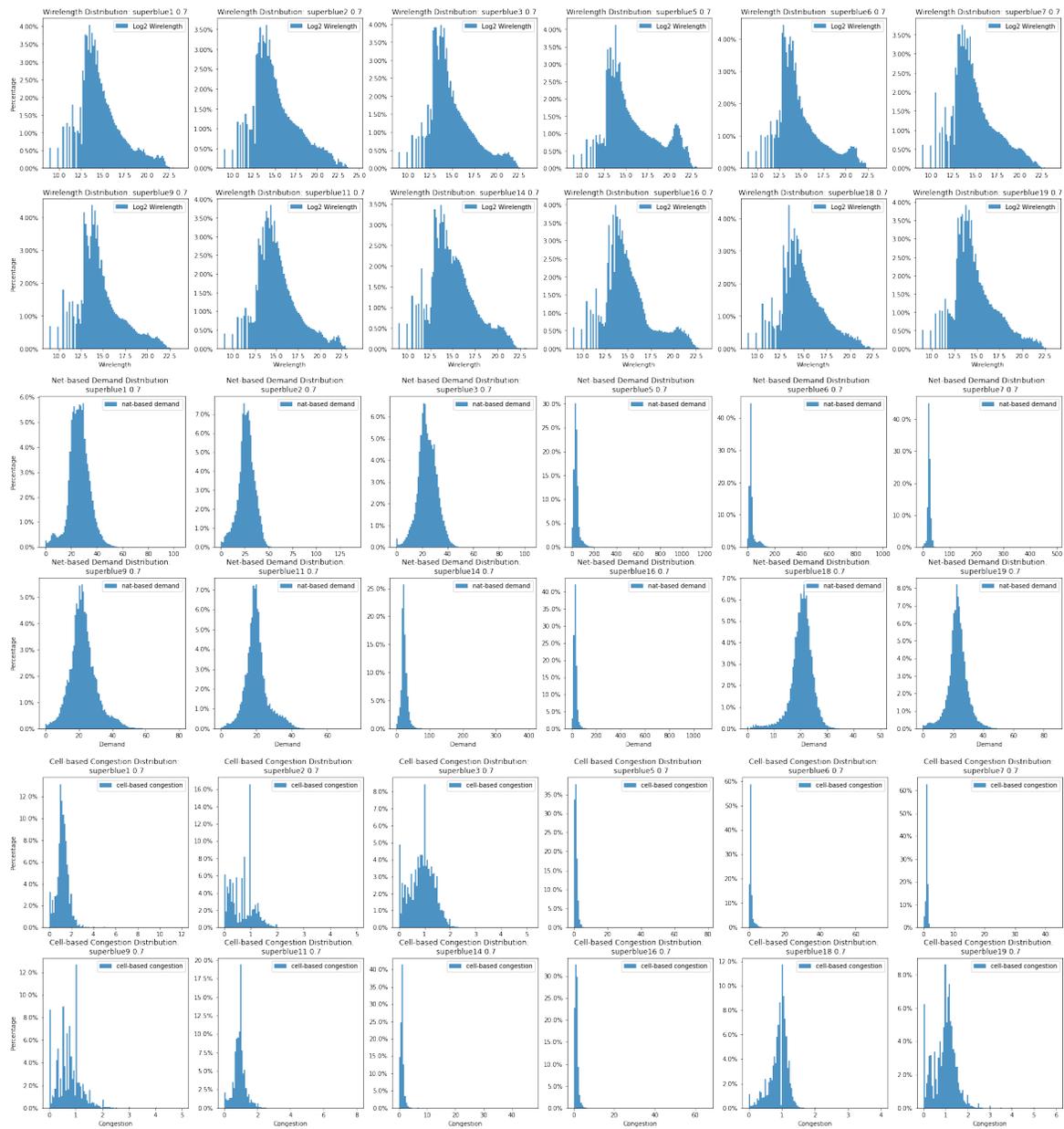

\centering
\includegraphics[width=0.90\textwidth]{figures/wirelength_distribution.pdf}
\includegraphics[width=0.90\textwidth]{figures/demand_distribution.pdf}
\includegraphics[width=0.90\textwidth]{figures/congestion_distribution.pdf}
\caption{{\small Net-based wirelength, net-based demand and cell-based congestion distributions of each design.}}
\label{fig:statistics_hist}
\end{figure}

\subsection{More experimental results}
\label{appendix:results}
\begin{table}[htbp]
\centering
\resizebox{1.0\columnwidth}{!}{
\begin{tabular}{c|ccccc}
\cmidrule{1-4}
& \multicolumn{3}{c}{\textbf{net-based wirelength regression}} \\
\cmidrule{2-4}
\textbf{Model} & \textbf{RMSE} $\downarrow$ & \textbf{MAE} $\downarrow$ & \textbf{Pearson} $\uparrow$ \\ 

\midrule
{\sf base E-HNN} & $1.818$ & $1.344$ & $0.731$ \\
{\sf base \textbf{\DEHNN}}  & $1.751$ & $1.269$ & $0.748$ \\
{\sf base \textbf{\DEHNN}+PD} & $1.731$ & $1.257$ & $0.754$ \\
{\sf base \textbf{\DEHNN}+PD+S. VN} & $1.724$ & $1.253$ & $0.758$ \\
{\sf full \DEHNN} & $1.689$ & $1.245$ & $0.770$ \\
\bottomrule

\end{tabular}

\begin{tabular}{|ccccc}
\toprule
\multicolumn{3}{|c}{\textbf{net-based demand regression}} \\
\midrule
\textbf{RMSE} $\downarrow$ & \textbf{MAE} $\downarrow$ & \textbf{Pearson} $\uparrow$ \\ 

\midrule
$9.228$ & $5.959$ & $0.591$ \\
$8.997$ & $5.764$ & $0.630$ \\
$8.765$ & $5.526$ & $0.647$ \\
$8.687$ & $5.519$ & $0.658$ \\
$8.381$ & $5.334$ & $0.683$ \\
\bottomrule

\end{tabular}

\begin{tabular}{|ccccc}
\toprule
\multicolumn{3}{|c}{\textbf{cell-based congestion classification}} \\
\midrule
\textbf{Precision} $\uparrow$ & \textbf{Recall} $\uparrow$ & \textbf{F\_score} $\uparrow$ \\ 

\midrule
$0.816$ & $0.864$ & $0.829$ \\
$0.824$ & $0.860$ & $0.840$ \\
$0.830$ & $0.869$ & $0.847$ \\
$0.832$ & $0.874$ & $0.851$ \\
$0.833$ & $0.876$ & $0.853$ \\
\bottomrule

\end{tabular}

}
\caption{Ablation Study: Full experimental results to show the improvements from \sf base E-HNN to \sf full \DEHNN.}
\label{table:ab}
\end{table}

\begin{table}[htbp]
\centering
\resizebox{1.0\columnwidth}{!}{
\begin{tabular}{c|ccccc}
\cmidrule{1-4}
& \multicolumn{3}{c}{\textbf{net-based wirelength regression}} \\
\cmidrule{2-4}
\textbf{Model} & \textbf{RMSE} $\downarrow$ & \textbf{MAE} $\downarrow$ & \textbf{Pearson} $\uparrow$ \\ 

\midrule
{\sf NetlistGNN with no PD} & $1.818$ & $1.344$ & $0.731$ \\
{\sf NetlistGNN+PD} & $1.773$ & $1.320$ & $0.740$ \\
\midrule
\textbf{Improvement} & 2.5\% & 1.8\% & 1.2\%\\
\midrule \midrule
{\sf GCN with no PD} & $1.809$ & $1.326$ & $0.735$ \\
{\sf GCN+PD} & $1.762$ & $1.276$ & $0.750$ \\
\midrule
\textbf{Improvement} & 1.9\% & 3.6\% & 5.2\%\\
\midrule \midrule
{\sf GATv2 with no PD} & $1.920$ & $1.401$ & $0.659$ \\
{\sf GATv2+PD} & $1.812$ & $1.330$ & $0.687$ \\
\midrule
\textbf{Improvement} & 0.7\% & 0.6\% & 1.6\%\\
\bottomrule

\end{tabular}

\begin{tabular}{|ccccc}
\toprule
\multicolumn{3}{|c}{\textbf{net-based demand regression}} \\
\midrule
\textbf{RMSE} $\downarrow$ & \textbf{MAE} $\downarrow$ & \textbf{Pearson} $\uparrow$ \\ 

\midrule
$9.237$ & $6.060$ & $0.592$ \\
$9.063$ & $5.839$ & $0.623$ \\
\midrule 
1.9\% & 3.6\% & 5.2\% \\
\midrule \midrule
$9.698$ & $6.453$ & $0.547$ \\
$9.321$ & $6.163$ & $0.570$ \\
\midrule
3.9\% & 4.5\% & 4.2\% \\
\midrule \midrule
$9.710$ & $6.392$ & $0.539$ \\
$9.342$ & $6.118$ & $0.561$ \\
\midrule
3.8\% & 4.3\% & 4.1\% & \\
\bottomrule

\end{tabular}

\begin{tabular}{|ccccc}
\toprule
\multicolumn{3}{|c}{\textbf{cell-based congestion classification}} \\
\midrule
\textbf{Precision} $\uparrow$ & \textbf{Recall} $\uparrow$ & \textbf{F\_score} $\uparrow$ \\ 

\midrule
$0.806$ & $0.855$ & $0.818$ \\
$0.812$ & $0.860$ & $0.831$ \\
\midrule
0.7\% & 0.6\% & 1.6\% \\
\midrule \midrule
$0.746$ & $0.837$ & $0.784$ \\
$0.761$ & $0.857$ & $0.802$ \\
\midrule
2.0\% & 2.4\% & 2.3\% \\
\midrule \midrule
$0.802$ & $0.856$ & $0.811$ \\
$0.810$ & $0.864$ & $0.835$ \\
\midrule
1.0\% & 1.0\% & 3.0\% \\
\bottomrule

\end{tabular}

} 
\caption{Ablation Study: the effect of using persistence diagrams (PDs) to three best-performing baselines. For each method, the 3rd row shows the percentage of improvement after using PD as part of the input features. Note that in the main text, the results we reported are those baselines+PD.}
\label{table:pd-ab}
\end{table}

In the main text, we have reported plots for ablation studies of different strategies used in our model. Here in Table \ref{table:ab} we report detailed numbers for all three tasks: net-based wirelength regression, net-based demand regression, and cell-based congestion classification across all the designs. 

Besides ablation study over our models, as we mentioned in main paper, adding the  persistence diagrams (PDs) features are beneficial for both our method and baselines. We have already shown the benefits for our model in the ablation studies in the main text and above. Here in Table \ref{table:pd-ab}, we show its benefits to the three best performing baseline models: NetlistGNN, GCN, GATv2. As we can see, using PDs improve all these models. Note that in the main text, when we report results of these baseline models, we are already reporting results with PDs. 

We show more detailed experimental results of single-design experiments for net-based wirelength regression in Table \ref{table:single-design-hpwl-full}, net-based regression in Table \ref{table:single-design-demand-full} and cell-based congestion classification in Table \ref{table:single-design-classify-full} for each design. We also show more results of cross-design experiments for net-based wirelength regression, net-based demand regression and cell-based congestion classification in Table \ref{table:cross-design-hpwl-full}. Similar to how we report the results in the main paper, we highlight the models' results with the best performance in red, and highlight the models (other than our models) that with second best performance as blue. 
Interestingly, in the single design case, it appears that the methods AllSet and NetlistGNN often perform the best among the baselines for net-length and net-based demand regression tasks, while for the cell-based congestion classification, other models (e.g, GATv2) sometimes perform the best among baselines. In cross-design experiments, GCN and GATv2 sometimes emerge as winners. 
In the cases that a baseline model is better than our model, we highlight that baseline model's results in red instead. 

\paragraph{Preliminary exploration with placement information.} 
In this paper, we focus on the case when our input do not have placement information (i.e, coordinates of cells). We also conducted some preliminary experiments to examine the effect of adding cell placement information to initial node features, and use placement information based partitioning instead of Metis. We first partition the nodes in netlist circuit (after placement) into $k$ bounding boxes $V_i$ with fixed width and height as $0.8$ millimeters. 
The number $k$ is depending on the width and height of each netlist circuit. 
With coordinates of cells added and bounding boxes based partitioning (for full-\DEHNN{}), we rerun both single-design and cross-design net-based demand regression experiments on superblue19. See Table \ref{table:placement_test} below. Placement information helps to improve the performance on average around 10\% for MAE and RMSE, and the improvement is more substantial in terms of Pearson Correlation. We aim to explore more about how to effectively leverage placement information in the future works.

\begin{table}[htbp]
\centering
\resizebox{0.95\columnwidth}{!}{
\begin{tabular}{ccccc}
\toprule
\textbf{Design} & \textbf{Model} & \textbf{RMSE} $\downarrow$ & \textbf{MAE} $\downarrow$ & \textbf{Pearson} $\uparrow$  \\
\midrule

\multirow{9}{*}{Superblue1} 
& GCN & 1.731 & 1.243 & 0.751 \\
& GATv2 & 1.742 & 1.253 & 0.742 \\
& HyperConv & 1.988 & 1.413 & 0.645 \\
& AllSet & \color{blue}1.700 & \color{blue}1.235 & \color{blue}0.757 \\
& HMPNN & 1.779 & 1.291 & 0.727 \\
& HNHN & 1.813 & 1.319 & 0.720 \\
& NetlistGNN & 1.761 & 1.276 & 0.735 \\
\cmidrule{2-5}%
& base \textbf{\DEHNN} & 1.774 & 1.274 & 0.731 \\
& \textbf{full \DEHNN} & \textbf{\color{red}1.657} & \textbf{\color{red}1.203} & \textbf{\color{red}0.771} \\
\cmidrule{2-5}%
& \textbf{Improvement} & 2.5$\%$ & 2.6$\%$ & 1.8$\%$ \\

\midrule
\multirow{9}{*}{Superblue2} 
& GCN & 2.011 & 1.487 & 0.714 \\
& GATv2 & 1.931 & 1.485 & 0.723 \\
& HyperConv & 2.217 & 1.661 & 0.641 \\
& AllSet & \color{blue}1.817 & \color{blue}1.353 & \color{blue}0.775 \\
& HMPNN & 1.981 & 1.480 & 0.723 \\
& HNHN & 1.992 & 1.488 & 0.724 \\
& NetlistGNN & 1.857 & 1.410 & 0.748 \\
\cmidrule{2-5}%
& base \textbf{\DEHNN} & 1.933 & 1.438 & 0.741 \\
& \textbf{full \DEHNN} & \textbf{\color{red}1.810} & \textbf{\color{red}1.344} & \textbf{\color{red}0.778} \\
\cmidrule{2-5}%
& \textbf{Improvement} & 0.4$\%$ & 0.7$\%$ & 0.4$\%$ \\

\midrule
\multirow{9}{*}{Superblue3} 
& GCN & 1.741 & 1.273 & 0.738 \\
& GATv2 & 1.716 & 1.270 & 0.748 \\
& HyperConv & 1.941 & 1.438 & 0.661 \\
& AllSet & \color{blue}1.685 & \color{blue}1.239 & \color{blue}0.759 \\
& HMPNN & 1.752 & 1.300 & 0.734 \\
& HNHN & 1.769 & 1.303 & 0.729 \\
& NetlistGNN & 1.689 & \color{blue}1.239 & 0.754 \\
\cmidrule{2-5}%
& base \textbf{\DEHNN} & 1.765 & 1.271 & 0.741 \\
& \textbf{full \DEHNN} & \textbf{\color{red}1.679} & \textbf{\color{red}1.234} & \textbf{\color{red}0.761} \\
\cmidrule{2-5}%
& \textbf{Improvement} & 0.4$\%$ & 0.4$\%$ & 0.3$\%$ \\

\midrule
\multirow{9}{*}{Superblue5} 
& GCN & 1.892 & 1.428 & 0.801 \\
& GATv2 & 1.909 & 1.438 & 0.793 \\
& HyperConv & 2.114 & 1.606 & 0.742 \\
& AllSet & \color{blue}1.842 & \color{blue}1.361 & \color{blue}0.810 \\
& HMPNN & 2.224 & 1.754	& 0.706 \\
& HNHN & 2.069 & 1.562 & 0.756 \\
& NetlistGNN & 1.915 & 1.406 & 0.790 \\
\cmidrule{2-5}%
& base \textbf{\DEHNN} & 1.881 & 1.393 & 0.799 \\
& \textbf{full \DEHNN} & \textbf{\color{red}1.795} & \textbf{\color{red}1.330} & \textbf{\color{red}0.822} \\
\cmidrule{2-5}%
& \textbf{Improvement} & 2.6$\%$ & 2.3$\%$ & 1.5$\%$ \\

\midrule
\multirow{9}{*}{Superblue6} 
& GCN & 1.811 & 1.341 & 0.740 \\
& GATv2 & 1.782 & 1.322 & 0.751 \\
& HyperConv & 1.928 & 1.435 & 0.702 \\
& AllSet & \color{blue}1.733 & \color{blue}1.260 & \color{blue}0.766 \\
& HMPNN & 1.851 & 1.360 & 0.730 \\
& HNHN & 1.859 & 1.365 & 0.727 \\
& NetlistGNN & 1.940 & 1.452 & 0.725 \\
\cmidrule{2-5}%
& base \textbf{\DEHNN} & 1.786 & 1.296 & 0.749 \\
& \textbf{full \DEHNN} & \textbf{\color{red}1.689} & \textbf{\color{red}1.233} & \textbf{\color{red}0.780} \\
\cmidrule{2-5}%
& \textbf{Improvement} & 2.5$\%$ & 2.1$\%$ & 1.8$\%$ \\

\midrule
\multirow{9}{*}{Superblue7} 
& GCN & 1.842 & 1.315 & 0.710 \\
& GATv2 & 1.847 & 1.371 & 0.699 \\
& HyperConv & 2.010 & 1.479 & 0.628 \\
& AllSet & \color{blue}1.753 & \color{blue}1.295 & \color{blue}0.733 \\
& HMPNN & 1.889 & 1.409 & 0.681 \\
& HNHN & 1.858 & 1.370 & 0.694 \\
& NetlistGNN & 1.801 & 1.316 & 0.721 \\
\cmidrule{2-5}%
& base \textbf{\DEHNN} & 1.760 & 1.291 & 0.701 \\
& \textbf{full \DEHNN} & \textbf{\color{red}1.719} & \textbf{\color{red}1.267} & \textbf{\color{red}0.747} \\
\cmidrule{2-5}%
& \textbf{Improvement} & 1.9$\%$ & 2.2$\%$ & 1.9$\%$ \\

\bottomrule
\end{tabular}

\begin{tabular}{ccccc}
\toprule
\textbf{Design} & \textbf{Model} & \textbf{RMSE} $\downarrow$ & \textbf{MAE} $\downarrow$ & \textbf{Pearson} $\uparrow$  \\
\midrule
\multirow{9}{*}{Superblue9} 
& GCN & 1.819 & 1.338 & 0.697 \\
& GATv2 & 1.853 & 1.381 & 0.683 \\
& HyperConv & 1.966 & 1.435 & 0.639 \\
& AllSet & \color{blue}1.684 & \color{blue}1.238 & \color{blue}0.750 \\
& HMPNN & 1.827 & 1.355 & 0.695 \\
& HNHN & 1.821 & 1.335 & 0.701 \\
& NetlistGNN & 1.805 & 1.325 & 0.701 \\
\cmidrule{2-5}%
& base \textbf{\DEHNN} & 1.765 & 1.287 & 0.720 \\
& \textbf{full \DEHNN} & \textbf{\color{red}1.663} & \textbf{\color{red}1.226} & \textbf{\color{red}0.757} \\
\cmidrule{2-5}%
& \textbf{Improvement} & 1.2$\%$ & 1.0$\%$ & 0.9$\%$ \\

\midrule
\multirow{9}{*}{Superblue11} 
& GCN & 1.831 & 1.356 & 0.694 \\
& GATv2 & 1.814 & 1.349 & 0.702 \\
& HyperConv & 1.986 & 1.468 & 0.633 \\
& AllSet & \color{blue}1.717 & \color{blue}1.274 & \color{blue}0.740 \\
& HMPNN & 1.830 & 1.369 & 0.696 \\
& HNHN & 1.834 & 1.353 & 0.697 \\
& NetlistGNN & 1.780 & 1.291 & 0.707 \\
\cmidrule{2-5}%
& base \textbf{\DEHNN} & 1.741 & 1.290 & 0.731 \\
& \textbf{full \DEHNN} & \textbf{\color{red}1.690} & \textbf{\color{red}1.257} & \textbf{\color{red}0.751} \\
\cmidrule{2-5}%
& \textbf{Improvement} & 1.6$\%$ & 1.3$\%$ & 1.5$\%$ \\

\midrule
\multirow{9}{*}{Superblue14} 
& GCN & 1.792 & 1.331 & 0.739 \\
& GATv2 & 1.794 & 1.361 & 0.738 \\
& HyperConv & 2.116 & 1.547 & 0.605 \\
& AllSet & \color{blue}1.756 & \color{blue}1.299 & \color{blue}0.750 \\
& HMPNN & 1.873 & 1.415 & 0.707 \\
& HNHN & 1.895 & 1.404 & 0.703 \\
& NetlistGNN & 1.812 & 1.363 & 0.731 \\
\cmidrule{2-5}%
& base \textbf{\DEHNN} & 1.816 & 1.336 & 0.730 \\
& \textbf{full \DEHNN} & \textbf{\color{red}1.728} & \textbf{\color{red}1.282} & \textbf{\color{red}0.760} \\
\cmidrule{2-5}%
& \textbf{Improvement} & 1.6$\%$ & 1.3$\%$ & 1.3$\%$ \\

\midrule
\multirow{9}{*}{Superblue16} 
& GCN & 1.763 & 1.265 & 0.741 \\
& GATv2 & 1.741 & 1.267 & 0.751 \\
& HyperConv & 2.047 & 1.446 & 0.639 \\
& AllSet & \color{blue}1.688 & \color{blue}1.207 & \color{blue}0.772 \\
& HMPNN & 1.896 & 1.362 & 0.701 \\
& HNHN & 1.816 & 1.312 & 0.731 \\
& NetlistGNN & 1.773 & 1.274 & 0.736 \\
\cmidrule{2-5}%
& base \textbf{\DEHNN} & 1.705 & 1.218 & 0.767 \\
& \textbf{full \DEHNN} & \textbf{\color{red}1.656} & \textbf{\color{red}1.194} & \textbf{\color{red}0.782} \\
\cmidrule{2-5}%
& \textbf{Improvement} & 1.9$\%$ & 1.1$\%$ & 1.3$\%$ \\

\midrule
\multirow{9}{*}{Superblue18} 
& GCN & 1.635 & 1.249 & 0.739 \\
& GATv2 & 1.701 & 1.246 & 0.714 \\
& HyperConv & 1.937 & 1.345 & 0.681 \\
& AllSet & 1.664 & 1.246 & 0.730 \\
& HMPNN & 1.769 & 1.335 & 0.686 \\
& HNHN & 1.752 & 1.314 & 0.696 \\
& NetlistGNN & \textbf{\color{red}1.625} & \textbf{\color{red}1.213} & \textbf{\color{red}0.752} \\
\cmidrule{2-5}%
& base \textbf{\DEHNN} & 1.653 & 1.263 & 0.735 \\
& \textbf{full \DEHNN} & 1.632 & 1.219 & 0.743 \\
\cmidrule{2-5}%
& \textbf{Improvement} & 0.4$\%$ & 0.5$\%$ & 1.2$\%$ \\

\midrule
\multirow{9}{*}{Superblue19} 
& GCN & 1.637 & 1.214 & 0.762 \\
& GATv2 & 1.634 & 1.198 & 0.765 \\
& HyperConv & 1.910 & 1.405 & 0.666 \\
& AllSet & \color{blue}1.582 & \color{blue}1.171 & \color{blue}0.783 \\
& HMPNN & 1.705 & 1.255 & 0.739 \\
& HNHN & 1.752 & 1.288 & 0.727 \\
& NetlistGNN & 1.635 & 1.205 & 0.765 \\
\cmidrule{2-5}%
& base \textbf{\DEHNN} & 1.641 & 1.208 & 0.763 \\
& \textbf{full \DEHNN} & \textbf{\color{red}1.557} & \textbf{\color{red}1.153} & \textbf{\color{red}0.792} \\
\cmidrule{2-5}%
& \textbf{Improvement} & 1.6$\%$ & 1.5$\%$ & 1.1$\%$ \\

\bottomrule
\end{tabular}
}
\caption{Average results of single-design net-based hpwl(wirelength) regression for each design, based on 4-fold cross validations. Last row ``Improvement'' refers to the improvement of our full \DEHNN{} model over the best baseline approach for each metric.}
\label{table:single-design-hpwl-full}
\end{table}

\begin{table}[htbp]
\centering
\resizebox{0.95\columnwidth}{!}{
\begin{tabular}{ccccc}
\toprule
\textbf{Design} & \textbf{Model} & \textbf{RMSE} $\downarrow$ & \textbf{MAE} $\downarrow$ & \textbf{Pearson} $\uparrow$  \\
\midrule

\multirow{11}{*}{Superblue1} 
& GCN & 6.469 & 4.979 & 0.595 \\
& GATv2 & 6.409 & 4.964 & 0.612 \\
& HyperConv & 6.662 & 4.951 & 0.224 \\
& AllSet & 6.100 & \color{blue}4.587 & 0.650 \\
& HMPNN & 6.770 & 5.206 & 0.541 \\
& HNHN & 6.394 & 4.825 & 0.610 \\
& Lin. Transformers & 7.991 & 6.046 & 0.089 \\
& NetlistGNN & \color{blue}6.039 & 4.623 & \color{blue}0.660 \\
\cmidrule{2-5}%
& base \textbf{\DEHNN} & 6.093 & 4.670 & 0.653 \\
& \textbf{full \DEHNN} & \textbf{\color{red}5.674} & \textbf{\color{red}4.263} & \textbf{\color{red}0.709} \\
\cmidrule{2-5}%
& \textbf{Improvement} & 6.0$\%$ & 7.1$\%$ & 7.4$\%$ \\

\midrule
\multirow{11}{*}{Superblue2} 
& GCN & 6.556 & 5.245 & 0.641 \\
& GATv2 & 6.736 & 5.288 & 0.616 \\
& HyperConv & 7.654 & 6.151 & 0.374 \\
& AllSet & 6.430 & 4.981 & 0.659 \\
& HMPNN & 6.982 & 5.501 & 0.579 \\
& HNHN & 6.699 & 5.217 & 0.625 \\
& Lin. Transformers & 8.251 & 6.356 & 0.313 \\
& NetlistGNN & \color{blue}6.259 & \color{blue}4.932 & \color{blue}0.682 \\
\cmidrule{2-5}%
& base \textbf{\DEHNN} & 6.399 & 5.035 & 0.663 \\
& \textbf{full \DEHNN} & \textbf{\color{red}5.966} & \textbf{\color{red}4.637} & \textbf{\color{red}0.718} \\
\cmidrule{2-5}%
& \textbf{Improvement} & 4.7$\%$ & 6.0$\%$ & 5.3$\%$ \\

\midrule
\multirow{11}{*}{Superblue3} 
& GCN & 5.789 & 4.512 & 0.612 \\
& GATv2 & 5.837 & 4.558 & 0.565 \\
& HyperConv & 6.468 & 5.149 & 0.265 \\
& AllSet & \color{blue}5.265 & \color{blue}4.014 & \color{blue}0.695 \\
& HMPNN & 6.022 & 4.713 & 0.572 \\
& HNHN & 5.686 & 4.338 & 0.631 \\
& Lin. Transformers & 7.046 & 5.493 & 0.264 \\
& NetlistGNN & 5.414 & 4.214 & 0.673 \\
\cmidrule{2-5}%
& base \textbf{\DEHNN} & 5.423 & 4.188 & 0.670 \\
& \textbf{full \DEHNN} & \textbf{\color{red}5.041} & \textbf{\color{red}3.837} & \textbf{\color{red}0.725} \\
\cmidrule{2-5}%
& \textbf{Improvement} & 4.3$\%$ & 4.4$\%$ & 4.3$\%$ \\

\midrule
\multirow{11}{*}{Superblue5} 
& GCN & \color{blue}27.169 & 14.867 & 0.504 \\
& GATv2 & 27.343 & 14.754 & 0.508 \\
& HyperConv & 29.563 & 15.817 & 0.107 \\
& AllSet & 27.881 & 14.632 & 0.490 \\
& HMPNN & 28.753 & 15.485 & 0.439 \\
& HNHN & 28.314 & 15.309 & 0.473 \\
& Lin. Transformers & 32.614 & 16.889 & 0.076 \\
& NetlistGNN & 27.586 & \color{blue}14.470 & \color{blue}0.515 \\
\cmidrule{2-5}%
& base \textbf{\DEHNN} & 27.205 & 14.129 & 0.536 \\
& \textbf{full \DEHNN} & \textbf{\color{red}26.684} & \textbf{\color{red}13.512} & \textbf{\color{red}0.565} \\
\cmidrule{2-5}%
& \textbf{Improvement} & 1.8$\%$ & 6.6$\%$ & 9.7$\%$ \\

\midrule
\multirow{11}{*}{Superblue6} 
& GCN & 21.963 & 12.692 & 0.607 \\
& GATv2 & 21.492 & 12.119 & 0.629 \\
& HyperConv & 25.615 & 13.356 & 0.128 \\
& AllSet & 17.945 & 10.156 & 0.759 \\
& HMPNN & 21.868 & 12.633 & 0.611 \\
& HNHN & \color{blue}17.735 & \color{blue}10.094 & \color{blue}0.767 \\
& Lin. Transformers & 28.807 & 14.583 & 0.119 \\
& NetlistGNN & 20.238 & 11.696 & 0.697 \\
\cmidrule{2-5}%
& base \textbf{\DEHNN} & 19.935 & 11.227 & 0.694 \\
& \textbf{full \DEHNN} & \textbf{\color{red}16.946} & \textbf{\color{red}9.680} & \textbf{\color{red}0.790} \\
\cmidrule{2-5}%
& \textbf{Improvement} & 4.4$\%$ & 4.1$\%$ & 3.0$\%$ \\

\midrule
\multirow{11}{*}{Superblue7} 
& GCN & 4.243 & 3.064 & 0.600 \\
& GATv2 & 4.403 & 3.247 & 0.561 \\
& HyperConv & 4.689 & 3.327 & 0.225 \\
& AllSet & 4.201 & 2.991 & 0.621 \\
& HMPNN & 4.527 & 3.246 & 0.541 \\
& HNHN & 4.458 & 3.165 & 0.557 \\
& Lin. Transformers & 5.245 & 3.752 & 0.139 \\
& NetlistGNN & \color{blue}4.115 & \color{blue}2.986 & \color{blue}0.634 \\
\cmidrule{2-5}%
& base \textbf{\DEHNN} & 4.110 & 2.957 & 0.631 \\
& \textbf{full \DEHNN} & \textbf{\color{red}3.971} & \textbf{\color{red}2.860} & \textbf{\color{red}0.662} \\
\cmidrule{2-5}%
& \textbf{Improvement} & 3.5$\%$ & 4.2$\%$ & 4.4$\%$ \\

\bottomrule
\end{tabular}

\begin{tabular}{ccccc}
\toprule
\textbf{Design} & \textbf{Model} & \textbf{RMSE} $\downarrow$ & \textbf{MAE} $\downarrow$ & \textbf{Pearson} $\uparrow$  \\
\midrule

\multirow{11}{*}{Superblue9} 
& GCN & 6.871 & 5.156 & 0.520 \\
& GATv2 & 6.893 & 5.139 & 0.512 \\
& HyperConv & 7.014 & 5.241 & 0.289 \\
& AllSet & \color{blue}6.184 & \color{blue}4.576 & \color{blue}0.640 \\
& HMPNN & 7.152 & 5.367 & 0.467 \\
& HNHN & 6.544 & 4.884 & 0.582 \\
& Lin. Transformers & 8.007 & 5.934 & 0.092 \\
& NetlistGNN & 6.511 & 4.796 & 0.589 \\
\cmidrule{2-5}%
& base \textbf{\DEHNN} & 2.990 & 2.228 & 0.696 \\
& \textbf{full \DEHNN} & \textbf{\color{red}5.685} & \textbf{\color{red}4.237} & \textbf{\color{red}0.709} \\
\cmidrule{2-5}%
& \textbf{Improvement} & 8.1$\%$ & 7.4$\%$ & 10.8$\%$ \\

\midrule
\multirow{11}{*}{Superblue11} 
& GCN & 5.693 & 4.224 & 0.502 \\
& GATv2 & 5.684 & 4.203 & 0.504 \\
& HyperConv & 6.243 & 4.523 & 0.105 \\
& AllSet & \color{blue}5.115 & \color{blue}3.792 & \color{blue}0.625 \\
& HMPNN & 5.974 & 4.402 & 0.418 \\
& HNHN & 5.277 & 3.882 & 0.592 \\
& Lin. Transformers & 6.576 & 4.678 & 0.034 \\
& NetlistGNN & 5.176 & 3.830 & 0.617 \\
\cmidrule{2-5}%
& base \textbf{\DEHNN} & 5.214 & 3.855 & 0.608 \\
& \textbf{full \DEHNN} & \textbf{\color{red}4.918} & \textbf{\color{red}3.677} & \textbf{\color{red}0.666} \\
\cmidrule{2-5}%
& \textbf{Improvement} & 3.9$\%$ & 3.0$\%$ & 6.6$\%$ \\

\midrule
\multirow{11}{*}{Superblue14} 
& GCN & 7.261 & 4.827 & 0.584 \\
& GATv2 & 7.370 & 4.825 & 0.545 \\
& HyperConv & 8.210 & 5.241 & 0.210 \\
& AllSet & 7.162 & 4.565 & 0.619 \\
& HMPNN & 7.687 & 4.992 & 0.540 \\
& HNHN & 7.327 & 4.693 & 0.597 \\
& Lin. Transformers & 8.853 & 6.168 & 0.176 \\
& NetlistGNN & \color{blue}6.872 & \color{blue}4.444 & \color{blue}0.642 \\
\cmidrule{2-5}%
& base \textbf{\DEHNN} & 6.874 & 4.453 & 0.639 \\
& \textbf{full \DEHNN} & \textbf{\color{red}6.533} & \textbf{\color{red}4.211} & \textbf{\color{red}0.684} \\
\cmidrule{2-5}%
& \textbf{Improvement} & 4.9$\%$ & 5.2$\%$ & 6.5$\%$ \\

\midrule
\multirow{11}{*}{Superblue16} 
& GCN & \textbf{\color{red}11.774} & 8.242 & 0.391 \\
& GATv2 & 12.853 & 8.283 & 0.377 \\
& HyperConv & 16.501 & 9.486 & 0.175 \\
& AllSet & 12.558 & \color{blue}7.837 & \color{blue}0.469 \\
& HMPNN & 13.539 & 8.582 & 0.312 \\
& HNHN & 12.720 & 8.082 & 0.446 \\
& Lin. Transformers & 14.020 & 8.827 & 0.003 \\
& NetlistGNN & 12.982 & 8.385 & 0.353 \\
\cmidrule{2-5}%
& base \textbf{\DEHNN} & 12.282 & 7.946 & 0.465 \\
& \textbf{full \DEHNN} & 11.867 & \textbf{\color{red}7.644} & \textbf{\color{red}0.520} \\
\cmidrule{2-5}%
& \textbf{Improvement} & - & 2.5$\%$ & 10.9$\%$ \\

\midrule
\multirow{11}{*}{Superblue18} 
& GCN & 3.061 & 2.262 & 0.681 \\
& GATv2 & 3.102 & 2.285 & 0.672 \\
& HyperConv & 4.013 & 2.915 & 0.255 \\
& AllSet & 3.057 & 2.294 & 0.674 \\
& HMPNN & 3.246 & 2.446 & 0.624 \\
& HNHN & 3.208 & 2.377 & 0.637 \\
& Lin. Transformers & 4.090 & 2.913 & 0.154 \\
& NetlistGNN & \color{blue}2.882 & \color{blue}2.173 & \color{blue}0.726 \\
\cmidrule{2-5}%
& base \textbf{\DEHNN} & 2.990 & 2.228 & 0.696 \\
& \textbf{full \DEHNN} & \textbf{\color{red}2.855} & \textbf{\color{red}2.136} & \textbf{\color{red}0.730} \\
\cmidrule{2-5}%
& \textbf{Improvement} & 0.9$\%$ & 1.7$\%$ & 0.6$\%$ \\

\midrule
\multirow{11}{*}{Superblue19} 
& GCN & 5.034 & 3.734 & 0.616 \\
& GATv2 & 4.949 & 3.691 & 0.636 \\
& HyperConv & 5.746 & 3.974 & 0.312 \\
& AllSet & 4.682 & \color{blue}3.474 & \color{blue}0.685 \\
& HMPNN & 5.294 & 3.980 & 0.571 \\
& HNHN & 5.063 & 3.750 & 0.620 \\
& Lin. Transformers & 6.315 & 4.565 & 0.127 \\
& NetlistGNN & \color{blue}4.683 & 3.520 & 0.681 \\
\cmidrule{2-5}%
& base \textbf{\DEHNN} & 4.946 & 3.720 & 0.632 \\
& \textbf{full \DEHNN} & \textbf{\color{red}4.429} & \textbf{\color{red}3.317} & \textbf{\color{red}0.723} \\
\cmidrule{2-5}%
& \textbf{Improvement} & 5.4$\%$ & 4.5$\%$ & 5.5$\%$ \\

\bottomrule
\end{tabular}
}
\caption{Results of single-design net-based demand regression for each design.}
\label{table:single-design-demand-full}
\end{table}

\begin{table}[htbp]
\centering
\resizebox{0.95\columnwidth}{!}{
\begin{tabular}{ccccc}
\toprule
\textbf{Design} & \textbf{Model} & \textbf{Precision} $\uparrow$ & \textbf{Recall} $\uparrow$ & \textbf{F\_score} $\uparrow$ \\
\midrule

\multirow{8}{*}{Superblue1} 
& GCN & 0.839 & 0.944 & 0.888 \\
& GATv2 & 0.867 & 0.944 & 0.904 \\
& HyperConv & 0.873 & 0.966 & 0.917 \\
& AllSet & \color{blue}0.880 & 0.955 & 0.916 \\
& HMPNN & 0.866 & 0.968 & 0.916 \\
& HNHN & 0.868 & \color{blue}0.969 & 0.916 \\
& Lin. Transformers & 0.853 & 0.941 & 0.895 \\
& NetlistGNN & 0.862 & 0.936 & \color{blue}0.920 \\
\cmidrule{2-5}%
& base \textbf{\DEHNN} & 0.876 & 0.967 & 0.920 \\
& \textbf{full \DEHNN} & \textbf{\color{red}0.885} & \textbf{\color{red}0.969} & \textbf{\color{red}0.925} \\
\cmidrule{2-5}%
& \textbf{Improvement} & 1.6$\%$ & 0.5$\%$ & 0.5$\%$ \\

\midrule
\multirow{8}{*}{Superblue2} 
& GCN & 0.741 & 0.657 & 0.697 \\
& GATv2 & \color{blue}0.782 & \color{blue}0.739 & \color{blue}0.760 \\
& HyperConv & 0.779 & 0.706 & 0.741 \\
& AllSet & 0.727 & 0.664 & 0.694 \\
& HMPNN & 0.730 & 0.587 & 0.649 \\
& HNHN & 0.718 & 0.633 & 0.670 \\
& Lin. Transformers & 0.752 & 0.530 & 0.621 \\
& NetlistGNN & 0.765 & 0.614 & 0.682 \\
\cmidrule{2-5}%
& base \textbf{\DEHNN} & 0.796 & 0.717 & 0.755 \\
& \textbf{full \DEHNN} & \textbf{\color{red}0.797} & \textbf{\color{red}0.767} & \textbf{\color{red}0.782} \\
\cmidrule{2-5}%
& \textbf{Improvement} & 1.2$\%$ & 9.7$\%$ & 2.9$\%$ \\

\midrule
\multirow{8}{*}{Superblue3} 
& GCN & 0.731 & 0.837 & 0.780 \\
& GATv2 & 0.768 & \textbf{\color{red}0.840} & 0.798 \\
& HyperConv & 0.770 & 0.815 & 0.792 \\
& AllSet & 0.728 & 0.773 & 0.747 \\
& HMPNN & 0.711 & 0.777 & 0.739 \\
& HNHN & 0.706 & 0.777 & 0.737 \\
& Lin. Transformers & 0.749 & 0.757 & 0.753 \\
& NetlistGNN & \color{blue}0.786 & 0.814 & \color{blue}0.799 \\
\cmidrule{2-5}%
& base \textbf{\DEHNN} & 0.791 & 0.819 & 0.805 \\
& \textbf{full \DEHNN} & \textbf{\color{red}0.817} & 0.816 & \textbf{\color{red}0.816} \\
\cmidrule{2-5}%
& \textbf{Improvement} & 0.7$\%$ & - & 2.1$\%$ \\

\midrule
\multirow{8}{*}{Superblue5} 
& GCN & 0.745 & 0.932 & 0.834 \\
& GATv2 & 0.783 & 0.923 & 0.848 \\
& HyperConv & 0.827 & 0.935 & 0.878 \\
& AllSet & 0.795 & \color{blue}0.939 & 0.861 \\
& HMPNN & 0.788 & 0.932 & 0.853 \\
& HNHN & 0.786 & 0.932 & 0.852 \\
& Lin. Transformers & 0.798 & 0.911 & 0.851 \\
& NetlistGNN & \color{blue}0.844 & 0.933 & \color{blue}0.885 \\
\cmidrule{2-5}%
& base \textbf{\DEHNN} & 0.842 & 0.938 & 0.887 \\
& \textbf{full \DEHNN} & \textbf{\color{red}0.852} & \textbf{\color{red}0.940} & \textbf{\color{red}0.894} \\
\cmidrule{2-5}%
& \textbf{Improvement} & 4.9$\%$ & 1.3$\%$ & 1.0$\%$ \\

\midrule
\multirow{8}{*}{Superblue6} 
& GCN & 0.837 & 0.921 & 0.877 \\
& GATv2 & \color{blue}0.876 & 0.920 & \color{blue}0.897 \\
& HyperConv & 0.851 & 0.916 & 0.891 \\
& AllSet & 0.817 & 0.940 & 0.874 \\
& HMPNN & 0.809 & \textbf{\color{red}0.965} & 0.879 \\
& HNHN & 0.815 & 0.947 & 0.875 \\
& Lin. Transformers & 0.833 & 0.906 & 0.868 \\
& NetlistGNN & 0.819 & 0.928 & 0.889 \\
\cmidrule{2-5}%
& base \textbf{\DEHNN} & 0.859 & 0.928 & 0.892 \\
& \textbf{full \DEHNN} & \textbf{\color{red}0.885} & 0.930 & \textbf{\color{red}0.906} \\
\cmidrule{2-5}%
& \textbf{Improvement} & 0.7$\%$ & - & 1.0$\%$ \\

\midrule
\multirow{8}{*}{Superblue7} 
& GCN & 0.839 & 0.980 & 0.904 \\
& GATv2 & 0.863 & \textbf{\color{red}0.981} & 0.920 \\
& HyperConv & \color{blue}0.899 & 0.967 & \color{blue}0.932 \\
& AllSet & 0.888 & 0.956 & 0.921 \\
& HMPNN & 0.874 & 0.962 & 0.916 \\
& HNHN & 0.875 & 0.955 & 0.913 \\
& Lin. Transformers & 0.792 & 0.870 & 0.891 \\
& NetlistGNN & 0.868 & 0.918 & 0.923 \\
\cmidrule{2-5}%
& base \textbf{\DEHNN} & 0.900 & 0.938 & 0.887 \\
& \textbf{full \DEHNN} & \textbf{\color{red}0.908} & 0.969 & \textbf{\color{red}0.937} \\
\cmidrule{2-5}%
& \textbf{Improvement} & 5.6$\%$ & - & 0.5$\%$ \\

\bottomrule
\end{tabular}

\begin{tabular}{ccccc}
\toprule
\textbf{Design} & \textbf{Model} & \textbf{Precision} $\uparrow$ & \textbf{Recall} $\uparrow$ & \textbf{F\_score} $\uparrow$ \\
\midrule

\multirow{8}{*}{Superblue9} 
& GCN & 0.684 & 0.556 & 0.613 \\
& GATv2 & \color{blue}0.719 & \color{blue}0.613 & \color{blue}0.666 \\
& HyperConv & 0.716 & 0.605 & 0.656 \\
& AllSet & 0.675 & 0.505 & 0.577 \\
& HMPNN & 0.612 & 0.418 & 0.495 \\
& HNHN & 0.664 & 0.447 & 0.529 \\
& Lin. Transformers & 0.649 & 0.551 & 0.596 \\
& NetlistGNN & \textbf{\color{red}0.778} & 0.568 & 0.656 \\
\cmidrule{2-5}%
& base \textbf{\DEHNN} & 0.740 & 0.647 & 0.690 \\
& \textbf{full \DEHNN} & 0.695 & \textbf{\color{red}0.653} & \textbf{\color{red}0.673} \\
\cmidrule{2-5}%
& \textbf{Improvement} & - & 6.5$\%$ & 1.1$\%$ \\

\midrule
\multirow{8}{*}{Superblue11} 
& GCN & 0.634 & 0.896 & 0.743 \\
& GATv2 & \color{blue}0.706 & 0.844 & 0.769 \\
& HyperConv & 0.644 & 0.910 & 0.755 \\
& AllSet & 0.620 & \textbf{\color{red}0.946} & 0.749 \\
& HMPNN & 0.627 & 0.927 & 0.748 \\
& HNHN & 0.628 & 0.897 & 0.738 \\
& Lin. Transformers & 0.671 & 0.691 & 0.680 \\
& NetlistGNN & 0.691 & 0.914 & \color{blue}0.787 \\
\cmidrule{2-5}%
& base \textbf{\DEHNN} & 0.677 & 0.863 & 0.759 \\
& \textbf{full \DEHNN} & \textbf{\color{red}0.719} & 0.850 & \textbf{\color{red}0.789} \\
\cmidrule{2-5}%
& \textbf{Improvement} & 1.1$\%$ & - & 0.3$\%$ \\

\midrule
\multirow{8}{*}{Superblue14} 
& GCN & 0.763 & \color{blue}0.891 & 0.822 \\
& GATv2 & \color{blue}0.834 & 0.889 & \color{blue}0.860 \\
& HyperConv & 0.830 & 0.886 & 0.857 \\
& AllSet & 0.809 & 0.863 & 0.835 \\
& HMPNN & 0.819 & 0.858 & 0.838 \\
& HNHN & 0.800 & 0.855 & 0.826 \\
& Lin. Transformers & 0.735 & 0.764 & 0.749 \\
& NetlistGNN & 0.827 & 0.870 & 0.787 \\
\cmidrule{2-5}%
& base \textbf{\DEHNN} & \textbf{\color{red}0.856} & 0.876 & 0.866 \\
& \textbf{full \DEHNN} & \textbf{\color{red}0.856} & \textbf{\color{red}0.902} & \textbf{\color{red}0.878} \\
\cmidrule{2-5}%
& \textbf{Improvement} & 3.7$\%$ & 10.6$\%$ & 2.1$\%$ \\

\midrule
\multirow{8}{*}{Superblue16} 
& GCN & 0.713 & 0.926 & 0.807 \\
& GATv2 & \color{blue}0.864 & \textbf{\color{red}0.928} & \color{blue}0.894 \\
& HyperConv & 0.855 & 0.833 & 0.844 \\
& AllSet & 0.844 & 0.827 & 0.833 \\
& HMPNN & 0.838 & 0.821 & 0.829 \\
& HNHN & 0.831 & 0.819 & 0.822 \\
& Lin. Transformers & 0.810 & 0.831 & 0.815 \\
& NetlistGNN & 0.779 & 0.912 & 0.865 \\
\cmidrule{2-5}%
& base \textbf{\DEHNN} & 0.874 & 0.823 & 0.847 \\
& \textbf{full \DEHNN} & \textbf{\color{red}0.895} & 0.910 & \textbf{\color{red}0.903} \\
\cmidrule{2-5}%
& \textbf{Improvement} & 8.5$\%$ & - & 1.0$\%$ \\

\midrule
\multirow{8}{*}{Superblue18} 
& GCN & 0.779 & 0.845 & 0.811 \\
& GATv2 & 0.798 & 0.848 & 0.822 \\
& HyperConv & 0.790 & 0.888 & 0.836 \\
& AllSet & 0.753 & 0.892 & 0.816 \\
& HMPNN & 0.763 & 0.888 & 0.821 \\
& HNHN & 0.749 & 0.881 & 0.810 \\
& Lin. Transformers & 0.775 & 0.852 & 0.812 \\
& NetlistGNN & \textbf{\color{red}0.868} & \textbf{\color{red}0.939} & \textbf{\color{red}0.902} \\
\cmidrule{2-5}%
& base \textbf{\DEHNN} & 0.798 & 0.890 & 0.842 \\
& \textbf{full \DEHNN} & 0.807 & 0.886 & 0.845 \\
\cmidrule{2-5}%
& \textbf{Improvement} & - & - & - \\

\midrule
\multirow{8}{*}{Superblue19} 
& GCN & 0.812 & 0.894 & 0.851 \\
& GATv2 & 0.857 & 0.899 & \color{blue}0.878 \\
& HyperConv & \color{blue}0.879 & 0.877 & \color{blue}0.878 \\
& AllSet & 0.870 & 0.866 & 0.868 \\
& HMPNN & 0.861 & 0.862 & 0.862 \\
& HNHN & 0.865 & 0.835 & 0.848 \\
& Lin. Transformers & 0.809 & 0.880 & 0.843 \\
& NetlistGNN & 0.869 & \textbf{\color{red}0.946} & 0.856 \\
\cmidrule{2-5}%
& base \textbf{\DEHNN} & 0.883 & 0.885 & 0.884 \\
& \textbf{full \DEHNN} & \textbf{\color{red}0.895} & 0.910 & \textbf{\color{red}0.903} \\
\cmidrule{2-5}%
& \textbf{Improvement} & 0.4$\%$ & - & 2.8$\%$ \\

\bottomrule
\end{tabular}
}
\caption{Results of single-design cell-based congestion classification for each design.}
\label{table:single-design-classify-full}
\end{table}

\begin{table}[htbp]
\vspace{-0.3cm}
\centering
\resizebox{1.0\columnwidth}{!}{
\begin{tabular}{c|ccc|ccc|ccc|ccc}
\toprule
& \multicolumn{3}{c|}{Single-Design without placement} & \multicolumn{3}{c|}{Single-Design with placement} & \multicolumn{3}{c|}{Cross-Design without placement} & \multicolumn{3}{c}{Cross-Design with placement} \\
\midrule
\textbf{Model} & RMSE & MAE & Pearson & RMSE (imp.)& MAE (imp.)& Pearson (imp.)& RMSE & MAE & Pearson & RMSE (imp.)& MAE (imp.)& Pearson (imp.)\\
\midrule
\large GCN & \large 5.034 & \large 3.734 & \large 0.616 & \large 4.495 (10.7\%) & \large 3.334 (10.7\%) & \large 0.717 (16.4\%) & \large 6.571 & \large 5.024 & \large 0.365 & \large 6.126 (6.8\%) & \large 4.709 (6.3\%) & \large 0.440 (20.5\%) \\
\large GATv2 & \large 4.949 & \large 3.691 & \large 0.636 & \large 4.382 (11.4\%) & \large 3.112 (15.7\%) & \large 0.758 (19.2\%) & \large 6.623 & \large 5.137 & \large 0.363 & \large 5.812 (10.8\%) & \large 4.695 (8.6\%)& \large 0.442 (21.8\%)\\
\midrule
\large full DE-HNN & \color{red}\large 4.429 & \color{red}\large 3.317 & \color{red}\large 0.723 & \color{red}\large 4.005 (9.6\%) & \color{red}\large 2.987 (10.0\%)& \color{red}\large 0.785 (8.6\%) & \color{red}\large 6.037 & \color{red}\large 4.670 & \color{red}\large 0.372 & \color{red}\large 5.795 (4.0\%)& \color{red}\large 4.337 (7.1\%)& \color{red}\large 0.452 (21.5\%) \\
\bottomrule
\end{tabular}
}
\caption{Results of net-based demand regression for \textbf{Superblue19}. For each metric, the \textbf{(imp.)} refers to the improvements when placement information added.}
\vspace{-0.8cm}
\label{table:placement_test}
\end{table}

\begin{table}[htbp]
\centering
\resizebox{1.0\columnwidth}{!}{
\begin{tabular}{ccccc}
\cmidrule{3-5}
& & \multicolumn{3}{c}{\textbf{Wirelength Regression}} \\
\midrule
\textbf{Design} & \textbf{Model} & \textbf{RMSE} $\downarrow$ & \textbf{MAE} $\downarrow$ & \textbf{Pearson} $\uparrow$  \\
\midrule

\multirow{8}{*}{Superblue19} 
& GCN & \color{blue}1.691 & \color{blue}1.276 & \color{blue}0.746  \\
& GATv2 & 1.717 & 1.281 & 0.737 \\
& Lin. Transformer & 2.159 & 1.588 & 0.521 \\
& NetlistGNN & 1.762 & 1.324 & 0.718 \\
\cmidrule{2-5}
& HyperConv & 2.390 & 1.788 & 0.558 \\
& Allset & 1.837 & 1.348 & 0.695 \\
& HMPNN & 1.785 & 1.335 & 0.710 \\
& HNHN & 1.754 & 1.333 & 0.701 \\
\cmidrule{2-5}
& base \textbf{\DEHNN} & 1.731 & 1.291 & 0.730 \\
& \textbf{full \DEHNN} & \color{red}1.677 & \color{red}1.242 & \color{red}0.754 \\
\cmidrule{2-5}
& \textbf{Improvement} & 1.9$\%$ & 2.6$\%$ & 1.8$\%$ \\

\bottomrule
\end{tabular}

\begin{tabular}{ccc}
\toprule
\multicolumn{3}{c}{\textbf{Demand Regression}} \\
\midrule
\textbf{RMSE} $\downarrow$ & \textbf{MAE} $\downarrow$ & \textbf{Pearson} $\uparrow$  \\
\midrule

6.571 & 5.024 & 0.365 \\
6.623 & 5.137 & 0.363 \\
6.564 & 4.819 & 0.086 \\
8.328 & 6.839 & \color{blue}0.367 \\
\cmidrule{1-3}
8.569 & 5.294 & 0.241 \\
\color{blue}6.120 & \color{blue}4.820 & 0.345 \\
6.979 & 5.356 & 0.306 \\
6.390 & 4.870 & 0.358 \\
\cmidrule{1-3}
6.778 & 5.085 & 0.337 \\
\color{red}6.037 & \color{red}4.670 & \color{red}0.372 \\
\cmidrule{1-3}
1.4$\%$ & 4.1$\%$ & 1.4$\%$ \\

\bottomrule
\end{tabular}

\begin{tabular}{ccc}
\toprule
\multicolumn{3}{c}{\textbf{Congestion Classification}} \\
\midrule
\textbf{Precision} $\uparrow$ & \textbf{Recall} $\uparrow$ & \textbf{F\_score} $\uparrow$  \\
\midrule

0.633 & 0.997 & 0.773  \\
0.630 & \color{red}0.999 & 0.765 \\
0.618 & 0.859 & 0.772 \\
0.647 & 0.953 & 0.771 \\
\cmidrule{1-3}
\color{blue}0.655 & 0.923 & \color{blue}0.778 \\
0.645 & 0.964 & 0.773 \\
0.633 & \color{red}0.999 & 0.773 \\
0.648 & 0.939 & 0.767 \\
\cmidrule{1-3}
0.653 & 0.990 & 0.774 \\
\color{red}0.660 & 0.986 & \color{red}0.780 \\
\cmidrule{1-3}
0.7$\%$ & - & 0.3$\%$ \\

\bottomrule
\end{tabular}

}
\caption{Results of cross-design net-based hpwl(wirelength) regression, net-based demand regression and cell-based congestion classification for different netlist design, including comparisons with other HNN models.}
\vspace{-0.5cm}
\label{table:cross-design-hpwl-full}
\end{table}

\end{document}